\documentclass{article}

\usepackage{amsthm,amssymb,url,amsmath}
\usepackage{mathtools}
\usepackage{pseudocode}
\usepackage[ruled,vlined]{algorithm2e}
\usepackage{subcaption}
\usepackage{graphicx}
\usepackage[noend]{algpseudocode}

\newtheorem{theorem}{Theorem}
\newtheorem{lemma}[theorem]{Lemma}
\newtheorem{corollary}[theorem]{Corollary}

\newcommand{\ptime}[3]{t^{^{(#1)}}_{_{#2,#3}}}

\begin{document}

\title{A framework for synchronizing a team of aerial robots in communication-limited environments\thanks{This paper is based in part on a conference paper by
the authors (\cite{dbanez2015icra}).}}

\author{J.M. D\'iaz-B\'a\~nez\thanks{Department of Applied Mathematics II, University of Seville, Spain.}\and 
L.E. Caraballo
\thanks{Department of Applied Mathematics II, University of Seville, Spain.} \and
M.A. Lopez
\thanks{Department of Computer Science, University of Denver, USA.} \and
S. Bereg
\thanks{Department of Computer Science, University of Texas at Dallas, USA.} \and
I. Maza
\thanks{Department of Systems Engineering and Automatic Control, University of Seville, Spain.} \and
A. Ollero
\thanks{Department of Systems Engineering and Automatic Control, University of Seville, Spain.}
}

\maketitle

\begin{abstract}
This paper  addresses a synchronization problem that arises when a team of aerial robots (ARs) need to communicate while performing assigned tasks in a cooperative scenario. Each robot has a limited communication range and flies within a previously assigned closed trajectory. When two robots are close enough, a communication link may be established, allowing the robots to exchange information. The goal is to schedule the flights such that the entire system can be synchronized for maximum information exchange, that is, every pair of neighbors always visit the feasible communication link at the same time. We propose an algorithm for scheduling a team of robots in this scenario and propose a robust framework in which the synchronization of a large team of robots is assured. The approach allows us to design a fault-tolerant system that can be used for multiple tasks such as surveillance, area exploration, searching for targets in a hazardous environment, and assembly and structure construction, to name a few.
\end{abstract}

{\bf Keywords:} Cooperative system, decentralized robots, synchronization, communication coordination, communication constraints.


\section{Introduction}
Multi-agent systems based on mobile robots  are ideally suited to perform cooperative missions in a cost-efficient manner (\cite{cao1997, bicchi2008}). These missions include monitoring, surveillance, structure assembly, and exploration, to name but a few. The use of a team of mobile robots presents advantages when considering mission execution time due to the parallelization of the tasks, fault tolerance and reduced uncertainty due to the possibility of overlapping information coming from the perception systems of different robots. These solutions are specially convenient for tasks which are too dangerous, or difficult for humans to perform.

The coordination between two or more robots in the team is important for team performance, to prevent collisions between robots occupying the same space, and to maintain proximity in order to perform a task cooperatively (\cite{rizzo2013signal-based, bekris2012safe, zlot2006market-based}).
Applications scenarios include:
 \begin{itemize}
\item Communication between robots exploring or monitoring an area, where the robots should be close enough in order to satisfy underlying communication range constraints.
\item Cooperative perception of the same target to decrease uncertainty (\cite{ollero2005}). Here the robots should be simultaneously in given positions and orientations  to have the target in the field of view of the sensors.
\item  Cooperative actuation. The robots should be coordinated to perform a task cooperatively, such as joint transportation of a load or joint assembly, avoiding waiting times.
  \end{itemize}

 High level communication is required in cooperative systems that perform a task in a decentralized manner. Also, the robustness of a cooperative system depends on the reliable communication between the agents involved. In many scenarios, direct communication among all the agents can not be guaranteed (large workspace with respect to the communication range of the system's members, for instance). In such scenarios, a cooperative task imposes communication constraints in which communication links are  available only when neighboring robots reside within a small communication range (\cite{burgard2005,sheng2006}). This is the situation we consider in this paper.

Consider a team of $n$ robots which are periodically traveling along predetermined closed trajectories while performing an assigned task. Each of the agents needs
to communicate information about its operation to other agents, but the communication interfaces have
a limited range. Hence, when two agents are within communication range, a communication link
is established, and information is exchanged. If two neighboring agents can exchange information periodically, we say that they are ``synchronized''. In this paper we consider the following synchronization problem: given the path geometries of a group of agents, schedule their movement along their trajectories so that every pair of neighboring agents is synchronized.

The synchronization problem arises naturally in missions of surveillance or monitoring (\cite{pasqualetti2012cooperative, acevedo_jint14}), in structure assembly while the robots are loading and placing parts in a structure (\cite{bernard2010}), or even in the exploration process looking for parts to be assembled as performed in the ARCAS project (\verb+http://www.arcas-project.eu+), to name but a few applications. In fact, its eventual solution may find many applications beyond the initial problems posed here.

The solution proposed in this paper ensures maximum information exchange of a heterogeneous team of robots with limited communication range.
Notice that the synchronization problem is even more relevant in aerial robots due to:
 \begin{itemize}
\item Rotorcraft robots (i.e. helicopters or multirotor systems) have very demanding energy requirements that limit the flight endurance. Then, hovering, which is very energy demanding,  waiting for other aerial robots to communicate or to perform cooperatively a task, should be minimized by means of synchronization.
\item The synchronization of fixed wing aircraft imposes more strict constraints when they should meet to interchange information due to the velocity of these aircrafts that may lead to communication losses when using short range communication devices.
  \end{itemize}

Additionally, the short communication range between aerial robots, and unmanned aerial vehicles in general, is interesting from the point of view of security. In fact, this short range communication may avoid communication jamming, which is a significant threat in the practical application of unmanned aerial vehicles (\cite{valavanis2015}). Therefore, in this paper we focus the synchronization problem in teams of aerial robots (ARs), but the results are applicable to general multi-robot systems.

To illustrate some of the issues arising in the synchronization problem consider a situation like that shown in Figure~\ref{fig:simple_pb} (3 robots with possible communication links between every pair of trajectories). Suppose, for simplicity, that the three ARs are flying with constant speed along trajectories with the same length in the same direction (clockwise or counterclockwise). How do we guarantee that each pair of ARs is synchronized, i.e., that each pair of neighbors arrives at the communication link at the same time? Under these conditions, it is easy to see that the three pairs of robots can not be synchronized if the sum of the lengths of the internal paths is less than the length of the total trajectory of a robot.

At first glance, one may think that the problem can be solved by adjusting speeds in order to force all meetings at the communication links or by waiting for a neighbor that has not arrived yet. Unfortunately, this approach is  impractical in many real scenarios when sudden changes in speed are not possible. In any case, both accelerations and time wasted by waiting for a neighbor reduces the performance of the system.

\begin{figure}[ht]
\centering
\includegraphics{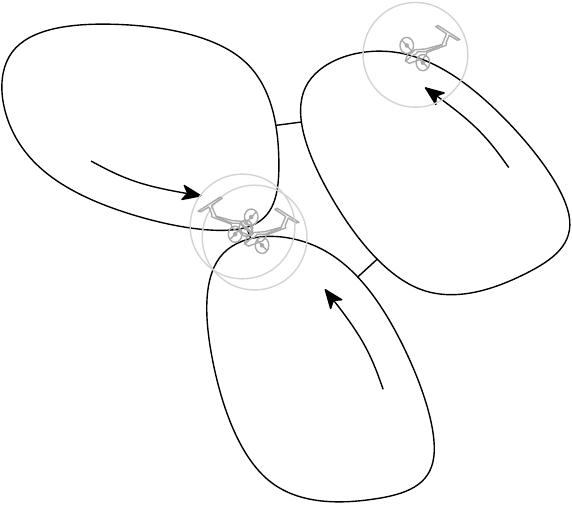}
\caption{Simple scenario of the synchronization problem.}
\label{fig:simple_pb}
\end{figure}

In next section we present previous work related to the coordination and communication of multi-robot systems. In Section~\ref{sec:problem} we define formally the problem. In Section \ref{sec:theoretical} we present conditions to guarantee the synchronization in a simplified model and we describe a scheduling so that the system is synchronized. Section \ref{sec:synchro_alg} presents an algorithm to compute the described scheduling using the previous theoretical results. A generalization of the results from the simplified model to more realistic scenarios is presented in Section \ref{gener}. Section \ref{simul} presents computational results for some cases. In Section \ref{sec:starvation} we introduce a new concept, that of \emph{starvation}, to describe a phenomenon characterized by the permanent loss of synchronization for one or more active ARs and propose some strategies to prevent this problem. Finally in Section \ref{sec:conclusion} we present the conclusions of our study and describe some future research.

\section{Related work}

As far as we know, the problem addressed in this paper has not been solved before. In this section, we give a brief review of related papers.
The distributed and decentralized coordination of a team of aerial robots under communication constraints has been studied in many specific cases, fundamentally in surveillance and monitoring missions. \cite{casbeer2006cooperative} and \cite{alexis2009coordination} present strategies to monitor cooperatively the perimeter of a forest fire using a team of small unmanned aerial vehicles with communication constraints. In this scenario, the team of robots operates on the fire perimeter and there is a single trajectory (that can change dynamically) for all the agents. Therefore, two robots moving in opposite directions always meet, so assigning different directions to the members of the team and changing direction alternately in every meeting event generates the necessary encounters between every pair of neighboring agents.

\cite{franchi2009} present a technique to explore unknown areas using a cooperative team of robots. The strategy is based in a data structure called Sensor-based Random Graph (SRG) where the members of team store environment data, the robots are constantly exploring to add new data to the SRG. When two robots meet at a point (fortuitously) they share information and generate new collision free motion plans to explore the potential unknown regions. Clearly, this technique does not solve our synchronization problem due to the meetings being fortuitous.

\cite{gil2015adaptive} consider the problem of satisfying communication demands in a multi-agent system where several robots cooperate on a task and a fixed subset of the agents act as mobile routers. The goal is to position the team of robotic routers to provide
communication coverage to the remaining client robots. The authors present an adaptive solution to allow for dynamic environments and variable client demands. The main difference with our problem is that we do not require a team of router robots to provide connectivity. In our setting the global communication is guaranteed with intermittent information exchange between neighboring robots while they perform there subtasks.

\cite{seokhoon2007search} propose a strategy called \emph{X Synchronization} (XS) to use autonomous, mobile and cooperative sensor nodes in search missions. The idea is to divide the area to be explored, a rectangle, for simplicity, into $n$ strips, one for each robot, and then to execute a \emph{lane based search}. The communication links are in the common borders between two consecutive strips. The data flow from the leftmost and rightmost agents to the center agent, and the decision control flow from the center to the sides of the team. This technique is interesting but, unfortunately it can not be applied to solve the problem for general workspaces (no strips possible) with irregular trajectories. The communication graph of this technique is just a line and the communication link is established between consecutive agents.

\cite{tardioli2010enforcing} present an integrated algorithm for task allocation and motion planning that keeps the connectivity of the cooperative team of robots. The members of the team have limited communication range but the distance between the agents does not admit disconnections. The tasks are assigned to the robots and the motion paths are computed under this restrictions. This strategy is not feasible for us because in our situation we are assuming large workspaces where it is not possible to establish a persistent communication network using as nodes the operating agents of the system.

In previous work, the authors have implemented and applied early versions of our algorithms for the coordination of aerial robots in some simple scenarios for area exploration and surveillance missions (\cite{acevedo_jint14,caraballo_icuas14}). In this paper, we expand and formally present the theoretical results and algorithms that guarantee the synchronization of a team of mobile robots with limited communication range and show how to use our approach in a fault-tolerant cooperative system. Simulations are  performed using a group of ARs to demonstrate the effectiveness of the proposed strategy.

\section{Problem formulation}\label{sec:problem}

The ingredients of the general problem considered here are the following:
 \begin{itemize}
\item A team of $n$ aerial robots need to share information while cooperating in the execution of a task in a decentralized way.
\item Each vehicle $V_i$ has a fixed communication range $r_i$ and flies with a constant altitude and speed in a specified closed trajectory $P_i$. The routes are disjoint, eliminating concerns about collisions.
\item A communication link (bridge) exists between the trajectories of vehicles $V_i$ and $V_j$ if and only if the minimum distance between the trajectories does not exceed the value $\min\{r_i,r_j\}$. The two vehicles may exchange information at all times when the distance between them is less than or equal to $\min\{r_i,r_j\}$.
\item We refer to the unique segment connecting two trajectories $P_i$ and $P_j$ with length of minimum distance not exceeding $\min\{r_i,r_j\}$ as the \emph{communication link} between the trajectories.
\item  We say that two robots are \emph{neighbors} if there exists a communication link between them; and
 two neighbors are \emph{synchronized} if they visit the communication link at the same time.
 \item A multi-robot system is \emph{synchronized} if  each pair of neighbors is synchronized.
 \item Given two synchronized robots  $V_i$ and $V_j$, the \emph{communication region} of $V_i$ and $V_j$,  is given by the two connected arcs $R_{ij}$ and $R_{ji}$ on $P_i$ and $P_j$ respectively such that when $V_i$ flies on $R_{ij}$ and $V_j$ flies on  $R_{ji}$, the distance between $V_i$ and $V_j$ is kept within $\min\{r_i,r_j\}$.
Two cases arise, depending on whether the two neighbors fly in opposite directions, one clockwise (CW) and the other counter-clockwise (CCW), or in the same direction. Figure~\ref{fig:commreg1} shows the latter case while Figure~\ref{fig:commreg2} illustrates the former.
Notice that the two cases affect differently the length of the arcs $R_{ij}$ and $R_{ji}$. We will later see that they also impact differently the robustness of the system.
\item In a synchronized system the vehicles establish communication when they enter the corresponding communication region without stopping.
   \end{itemize}

\begin{figure}[!h]
\centering
\begin{subfigure}{.5\textwidth}
\centering
\includegraphics[scale=.7 , page=1]{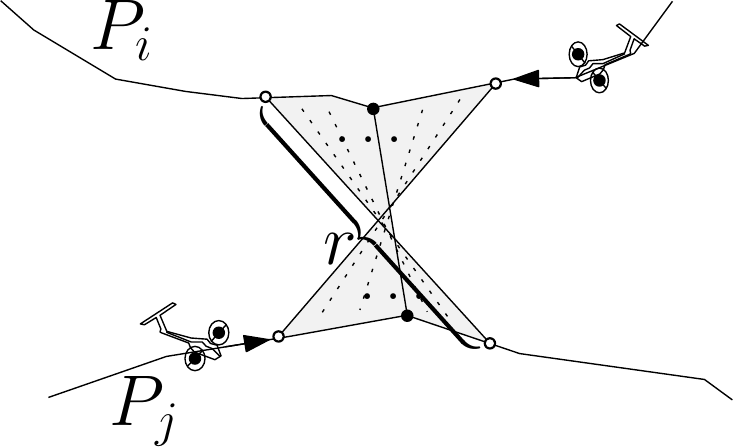}
    \caption{One direction}
    \label{fig:commreg1}
\end{subfigure}%
\begin{subfigure}{.5\textwidth}
\centering
\includegraphics[scale=.7 , page=2]{communication_region.pdf}
    \caption{Two directions}
\label{fig:commreg2}
\end{subfigure}
\caption{Representation of the communication region between two robots $V_i$ and $V_j$ (gray region): (a) ARs flying in the same direction and (b) ARs flying in opposite directions. Notice that both figures use the same communication range ($r=min\{r_i, r_j\}$).}
\end{figure}

\emph{{\bf The Synchronization Problem}: Given a set of $n$ Aerial Robots (ARs), each performing part of a cooperative task within a closed trajectory, and exchanging information with a fixed communication range, schedule the flight trajectories such that the number of synchronized AR pairs is maximized. When this number equals the number of communication links, the system is fully synchronized.}
\vspace{.5cm}

First, observe that under a general
model, even the synchronization between two robots cannot be
guaranteed. Consider, for example, a system consisting of only two ARs, each flying at constant speed with a very small communication range. Then, if the ratio of the two trajectory lengths is not
rational, 
a synchronized flight is not possible. Typically, the methods used in the literature achieve synchronization by changing the speeds of the ARs by small amounts, i.e., they allow for the possibility of one vehicle ``waiting" for the other. Unfortunately, this simple approach is only feasible for two vehicles. For a team of cooperative ARs,
a more delicate theoretical study is required.

The following questions need to be answered in order to implement an efficient and robust multi-robot coordinated system: (i) When can a cooperative multi-robot system be synchronized without changing the robot speeds? and (ii) in case of a robot failure, can the schedules be slightly altered in order to complete the global task in a new synchronized system?

This paper aims to answer the above two questions. Based on theoretical results on a simplified model, we propose an algorithm that is efficient and robust in the face of catastrophic robot failures where the synchronization of a large team of ARs is assured. Although the application scenario highlighted here is the exploration of an unknown environment by a system of ARs,  the proposed concepts and methodologies can be useful in other multi-robot applications.

%
%

\section{Theoretical results in a simplified model}\label{sec:theoretical}

The general methodology we are proposing is to first obtain strong results on a simplified, albeit not entirely practical model, and then to adapt the newly acquired theoretical knowledge to more general and realistic models.
Accordingly, we discuss how to extend the approach for the simple model to others models in Section~\ref{sec:synchro_alg} where we consider, for example, heterogeneous vehicles, non-circular routes, etc.

Let us consider a simplified model for which the basic results can be stated.
In this simplified model all aerial robots move in equal and pairwise disjoint circular trajectories at the same speed with the same communication range. Let $C_1,C_2,\dots,C_n$ be the pairwise disjoint unit circles representing the flight trajectories of the $n$ robots and let $r$ be the communication range  (normalized to be consistent with the circles of unit radius).

In the system two AR's can potentially share information if the distance between their corresponding trajectories is at most $r$ (see Figure~\ref{fig:ARs}). In order to model the ensuing communication constraints we define the \emph{communication graph} of the system {\em with respect to range $r$\/} as a planar graph $G(r)=(V,E(r))$ whose vertices are the circle centers and whose edges connect two centers if their distance is less or equal to $2+r$ (see Figure~\ref{fig:2angles}). We denote by $(i,j)$ the edge that connects the centers of the circles $C_i$ and $C_j$.

In the analysis that follows, it will be convenient to denote the position of an AR by the angle, measured from the positive horizontal axis, that its position makes on the unit circle (see Figure~\ref{fig:ARs}). Fix an arbitrary edge $(i,j)$ of $E(r)$. The \emph{link position of $i$ with respect to $j$}, denoted by $\phi_{ij}$, is the angle at which the AR in $C_i$ is closest to $C_j$. Clearly, if $\phi_{ij}$ is defined, so is $\phi_{ji}$ and $\phi_{ji}=\pi + \phi_{ij}$.  
 Then, two neighbors $i$ and $j$ are \emph{synchronized} if both arrive at the same time to their link positions $\phi_{ij}$ and $\phi_{ji}$, respectively. Notice that the number of synchronized pairs is bounded by $|E(r)|$. Because of our model assumptions, two synchronized neighbors ``meet" each other repeatedly every $2\pi$ units of travel.

\begin{figure}[h]
    \centering
    \includegraphics[width=.4\textwidth, trim = 4cm 5cm 4cm 6cm, clip=true]{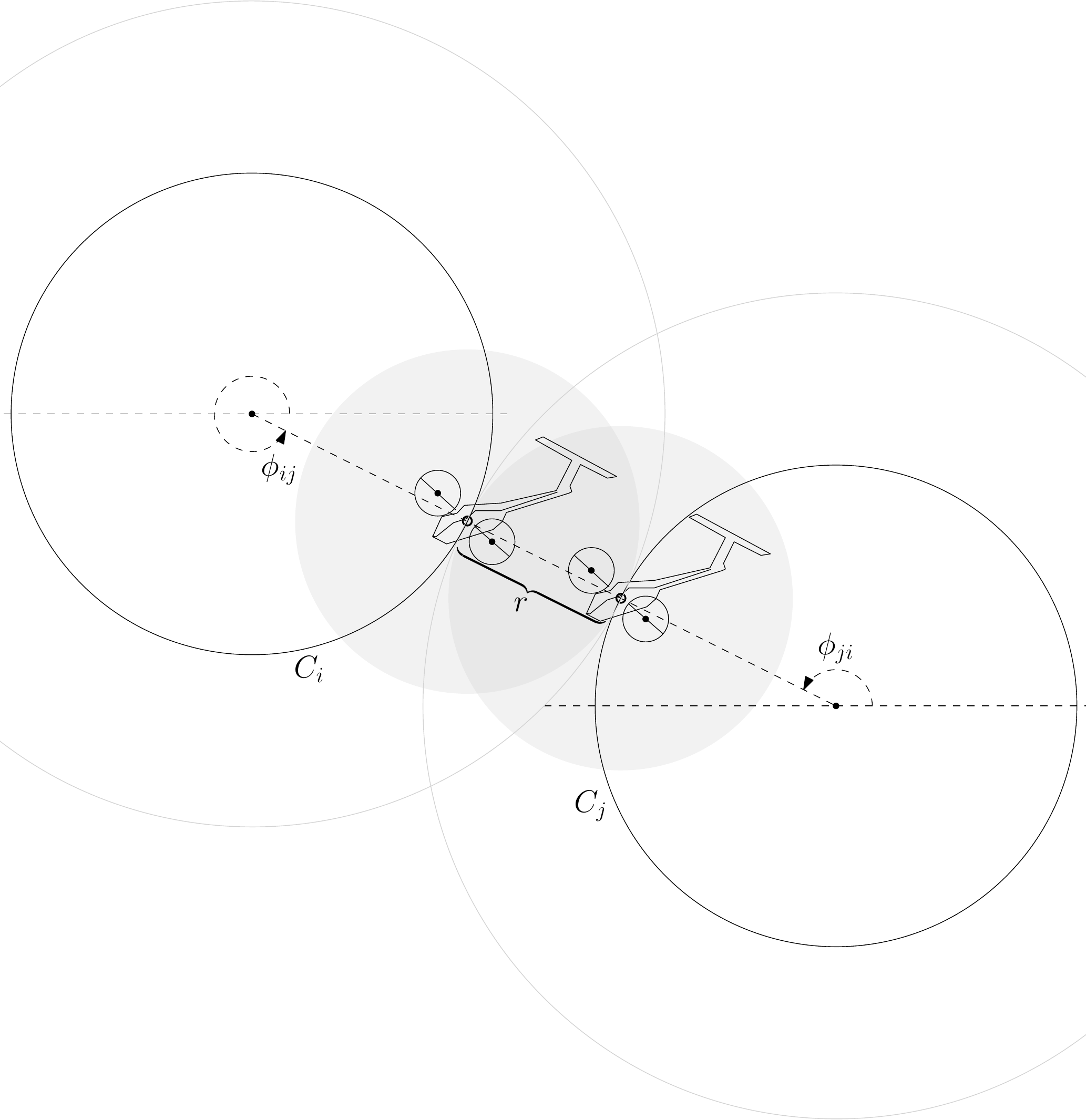}
    \caption{The simple model. Robots $i$ and $j$ can share information if the distance between $C_i$ and $C_j$ is less or equal to $r$. $\phi_{ij}$ (resp. $\phi_{ji}$) is the angle at which $i$ (resp. $j$) is closest to $j$'s trajectory (resp. $i$'s trajectory).}
    \label{fig:ARs}
\end{figure}

\begin{figure}[h]
    \centering
    \includegraphics[width=.45\textwidth]{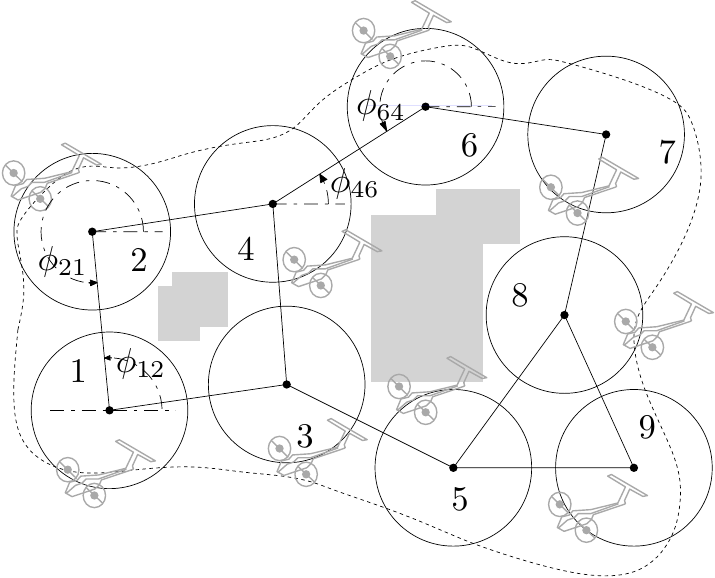}
    \caption{Representation of a set $C_1,C_2,\dots,C_9$ of unit circles and the underlying communication graph.}
    \label{fig:2angles}
\end{figure}

Since the ARs move at constant speed, it suffices to know the starting position and movement direction of an AR in order to compute its position at any time. Thus, we can define a  \emph{flight schedule} of a team of ARs as the set of starting positions and directions of the ARs involved. If the number of synchronized pairs in a flight schedule is $|E(r)|$, we say that the team is \emph{fully synchronized}.

\subsection{Flying in the same direction}
In this subsection we derive conditions that ensure full synchronization of the team when all ARs are flying in the same direction (CW or CCW). Without loss of generality, we assume that the direction of flight is counter-clockwise. We want to compute the starting angles $\alpha_1,\alpha_2,\dots,\alpha_n$ corresponding to the vehicles $V_1,V_2,\dots,V_n$ such that if $(i,j)\in E(r)$ then $V_i$ and $V_j$ reach $\phi_{ij}$ and $\phi_{ji}$, respectively, at the same time. The following result gives us the key to establish sufficient conditions for the existence of a synchronization strategy:

\begin{lemma}\label{2angles}
Let $G(r)=(V,E(r))$ be a communication graph in the simplified model.
If $(i,j)\in E(r)$, $(i,k)\in E(r)$ and $V_i$ is synchronized with $V_j$ and $V_k$ then $\alpha_j=\alpha_k$.
\end{lemma}

\begin{proof}
Without loss of generality, assume $\alpha_i=\phi_{ij}$ and $\alpha_j=\phi_{ji}$, i.e., both $V_i$ and $V_j$ start at their link position, allowing them to monitor each other. We now need to compute $\alpha_k$ so that when $V_i$ is at $\phi_{ik}$, $V_k$ is at $\phi_{ki}$ (see Figure~\ref{fig:lemma_kij}). Thus,
\begin{eqnarray*}
  \alpha_k &=& \phi_{ki}-(\phi_{ik}-\phi_{ij}) \\
           &=& (\phi_{ki}-\phi_{ik})+\phi_{ij}= \pi + \phi_{ij} = \phi_{ji} \\
           &=& \alpha_j
\end{eqnarray*}
\begin{figure}[h]
    \centering
    \includegraphics[width=.4\textwidth]{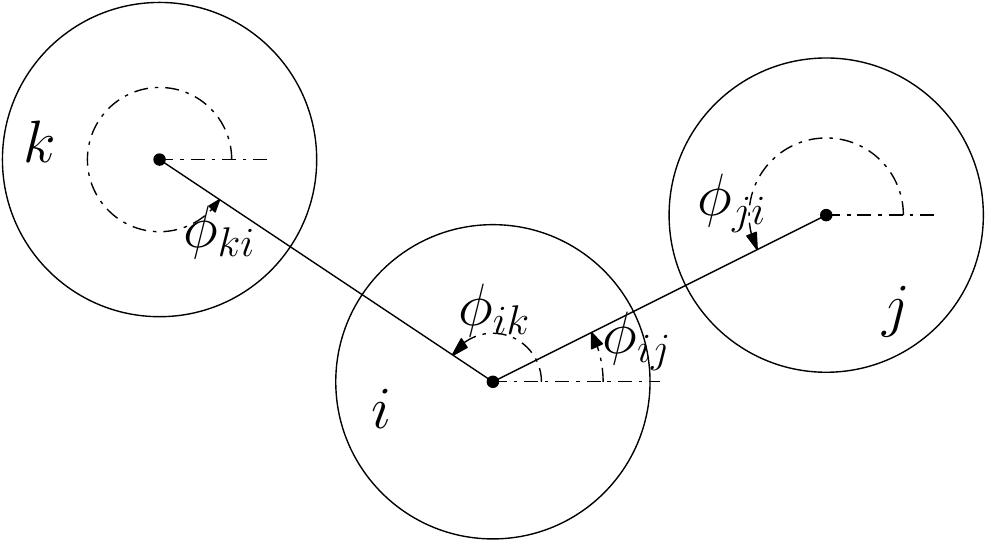}
    \caption{If $V_i$ is synchronized with $V_j$ and $V_k$ then $\alpha_{k}=\alpha_{j}$.}
    \label{fig:lemma_kij}
\end{figure}
\end{proof}

Using Lemma~\ref{2angles} we derive the following corollary:

\begin{corollary}\label{cor:not_oddCycle}
Let $G(r)=(V,E(r))$ be the communication graph in the simplified model.
Cycles of odd length in $G(r)$ cannot be synchronized.
\end{corollary}

\begin{proof}
  Recall that if $(i,j)\in E(r)$ then $\phi_{ij}\ne\phi_{ji}$ and, consequently, $\alpha_i\ne\alpha_j$. Suppose that $G$ contains an odd cycle $\langle V_{i_0},V_{i_1},\ldots,V_{i_{2c}},V_{i_0}\rangle$ that can be synchronized with starting angles $\alpha'_0,\alpha'_1,\ldots,\alpha'_{2c}$, respectively. Since the two neighbors of any vertex in the cycle must share the same starting angle (Lemma~\ref{2angles}), it follows that, as we move around the cycle, starting angles alternate between $\alpha'_0$ and $\alpha'_1$. This forces two neighbors (say $V_{i_0}$ and $V_{i_{2c}}$) to have the same starting angle, a contradiction.
\end{proof}

Now we are ready to prove the main result for the simplified model when the robots move in the same direction.

\begin{theorem}
Let $G(r)=(V,E(r))$ be communication graph of $n$ robots flying in the same direction and with the same speed and communication range $r$. The system can be fully synchronized if and only if $G(r)$ is bipartite. Moreover, the condition $\alpha_{i} = \pi + \alpha_{j}$ for every $(i,j)\in E(r)$ ensures synchronization of the team.
\end{theorem}
\begin{proof}
If $G(r)$ is not bipartite then it contains an odd cycle and by Corollary~\ref{cor:not_oddCycle}, the team cannot be fully synchronized. If $G(r)$ is bipartite then it is two-colorable. After a white-black coloration of the graph we set an arbitrary position $\beta$ for all white ARs and set $\pi + \beta$ for the black ones. Let $(i,j)$ be any edge of $E(r)$. Then $V_i$ is white and $V_j$ is black or vice versa. Since $\phi_{ji}=\pi+\phi_{ij}$, the starting positions of $V_i$ and $V_j$ are antipodal.
(see Figure~\ref{fig:bipartite-synch}).
In fact, if $V_i$ and $V_j$ start in antipodal positions ($\alpha_j=\pi+\alpha_i$), and maintain the same speed and direction, they will be occupy antipodal positions and, consequently, will repeatedly reach the link position at the same time.


\begin{figure}[h]
    \centering
    \includegraphics[scale=0.6]{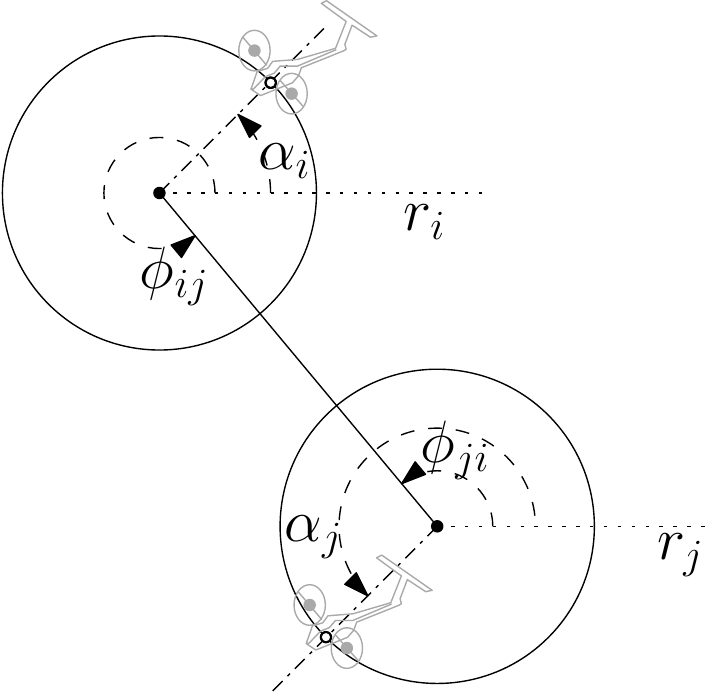}
    \caption{This Figure illustrates the start position of two synchronized neighbors in $G(r)$.}
    \label{fig:bipartite-synch}
\end{figure}

\end{proof}

\subsection{Flying in opposite directions}
In this subsection we derive conditions that ensure full synchronization of the team when the members of every pair of neighbors fly in opposite directions. Note that such a model can only be deployed on bipartite communication graphs. In this case, the partition of the ARs according to direction of flight (CW or CCW) corresponds exactly to the partite sets of the bipartite communication graph.
This model is important because it will allow us to provide a certain degree of robustness while preserving full synchronization (see subsection~\ref{sec:robutness}).


Let $(i,j)$ be an edge in $E(r)$ and $l_{ij}$ be the supporting line of the edge $(i,j)$. We denote by $\beta_{ij}$ the angle of the line $l_{ij}$ measured from the positive horizontal axis.

\begin{lemma}\label{lem:synch_twoDir}
Let $G(r)=(V,E(r))$ be the communication graph in the simplified model.
 Let $(i,j)$ be an edge in $E(r)$ and consider the ARs  $V_i$ and $V_j$ moving in opposite directions. Then, $V_i$ and $V_j$ are synchronized iff $\alpha_j = 2 \beta_{ij} - \alpha_i \pm \pi$.
\end{lemma}
\begin{proof}
The angles formed by positions $\alpha_j$ and $\alpha_i$ with $l_{ij}$ are equal. Let $\alpha_i'$ be the symmetrical position of $\alpha_i$ obtained by reflecting $\alpha_i$ with respect to $\l_{ij}$. Translating the position $\alpha_j$ from $C_j$ to $C_i$ we see that $\alpha_j$ and $\alpha_i'$ are antipodal, i.e., they differ by $180^{\circ}$ (see Figure~\ref{fig:proof-diff-dir}).
\begin{align}
\alpha_i'&=\beta_{ij}-(\alpha_i-\beta_{ij}) \nonumber\\
&=2\beta_{ij}-\alpha_i \nonumber
\end{align}
To derive $\alpha_j$ we simply add or subtract $\pi$ to $\alpha_i'$.
\end{proof}
\begin{figure}[h]
\centering
\includegraphics[scale=.7]{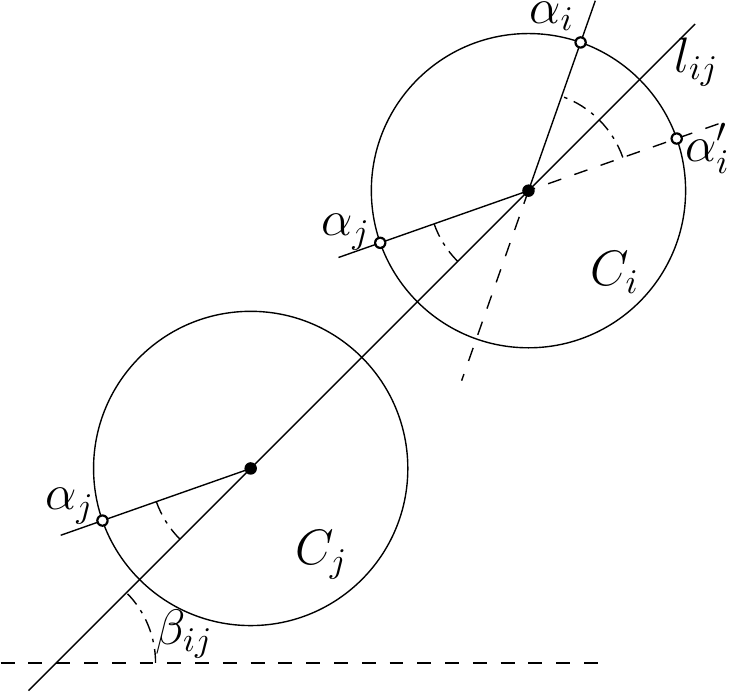}

\caption{The position $\alpha_i'$ is the symmetric of $\alpha_i$ respect $l_{ij}$. The position $\alpha_j$ is the opposite of $\alpha_i'$.}
\label{fig:proof-diff-dir}
\end{figure}
\begin{theorem}\label{thm:opposite_cycles}
Let $G(r)=(V,E(r))$ be a communication graph in the simplified model.
Suppose that $G(r)$ is bipartite and that every pair of neighbors are scheduled to fly in opposite directions. Then the system is fully synchronized if and only if all even cycles
$\langle V_{i_1},V_{i_2},\ldots,V_{i_{2k}},V_{i_1}\rangle$ in $G_r$  
satisfy:
$$\beta_{i_1i_2}-\beta_{i_2i_3}+\beta_{i_3i_4}-\beta_{i_4i_5}+\dots+\beta_{i_{2k-1}i_{2k}}-\beta_{i_{2k}i_1}=2m\pi$$
with $m\in\mathbb{N}$.

\end{theorem}
\begin{proof}
Let $\alpha_{i_1}$ be the starting position of $V_{i_1}$. 
We can use Lemma~\ref{lem:synch_twoDir} to compute the starting position of the remaining ARs (arbitrarily using $+\pi$ instead of $-\pi$ as described in the lemma):
\begin{align}
\alpha_{i_2}&= 2\beta_{i_2i_1}-\alpha_{i_1} + \pi \nonumber\\
\alpha_{i_3}&= 2\beta_{i_3i_2}-2\beta_{i_2i_1}+\alpha_{i_1} \nonumber\\
\alpha_{i_4}&= 2\beta_{i_4i_3}-2\beta_{i_3i_2}+2\beta_{i_2i_1}-\alpha_{i_1}+\pi \nonumber\\
\vdotswithin{\alpha_{i_{2k}}}&\vdotswithin{2\beta_{i_{2k}i_{2k-1}}-2\beta_{i_{2k-1}i_{2k-2}}+}\vdotswithin{\quad\dots\quad} \nonumber\\
\alpha_{i_{2k}}&=2\beta_{i_{2k}i_{2k-1}}-2\beta_{i_{2k-1}i_{2k-2}}+\quad\dots\quad-\alpha_{i_1}+\pi \nonumber\\
\alpha_{i_1}&=2\beta_{i_1i_{2k}}-2\beta_{i_{2k}i_{2k-1}}+\quad\dots\quad-2\beta_{i_2i_1}+\alpha_{i_1} \label{eq:cycle_inverted}
\end{align}
From Equation~\ref{eq:cycle_inverted} we obtain $$2\beta_{i_1i_{2k}}-2\beta_{i_{2k}i_{2k-1}}+\quad\dots\quad-2\beta_{i_2i_1}=0\ (\mbox{mod } 2\pi) $$
\end{proof}

\subsection{Robustness}\label{sec:robutness}
We now address the issue of robustness of the synchronized system.
Imagine that one member $V$ of the team becomes unavailable because of failure or because it needs to leave the system temporarily (for instance, to refuel). First, let us consider the latter scenario. We want to minimize the detrimental effect of the departing AR on global system performance. Then, a simple strategy consists of repeatedly ``swapping with a neighbor'' until $V$  leaves the system as illustrated in Figure~\ref{fig:swappingToOut}.
The best moment to make a swap with a neighbor is when the ARs involved arrive to a link position (this is, indeed,  another advantage of a synchronized system). However, if all ARs fly in the same direction,  the kinematic constraints of the aerial vehicles can prevent making a swap with a neighbor (see Figure~\ref{fig:swappingSameDir}). However, the swap is not a problem if neighbors are flying in opposite directions. In this case, the neighboring robots can interchange their routes with a smooth maneuver (see Figure~\ref{fig:swappingDiffDir}). Consider now the case of  catastrophic failure, where one (or more) robots fail and
can no longer move. Under the assumption that
robots which fail do not block live robots, we can solve easily the problem while maintaining the set of trajectories.
In this case, one or more live neighbors can assume the tasks of the inoperative robots by using the same routes. Moreover, a new robot can be inserted in the system by using the swap strategy as explained above. In both cases, the system can be restored  by means of local changes and this ensures the
robustness of the approach.

The following theorem establishes a non-redundancy property of ARs occupying the trajectory of a fallen
neighbor
by assuming the
following conditions:
(i) All ARs require the same time to complete a tour in any trajectory, and
 (ii) The time required to switch trajectories (at a link position) is negligible, which should be
 possible by slightly increasing or decreasing the speed.

\begin{theorem}
In a synchronized system allowing swaps in case of failures, each trajectory is occupied at most by one AR.
\end{theorem}
\begin{proof}
Assume, for the sake of contradiction, that there exists a trajectory $P_i$ occupied by two ARs $j$ and $k$
who entered $P_i$ trough $\phi_{ji}$ and $\phi_{ki}$, respectively. An AR may enter a neighbor's trajectory
when it arrives to a link position and the neighbor assigned to the trajectory is not there. Thus, $j$
and $k$ could not have entered $P_i$ at the same time because $i$ cannot be at $\phi_{ij}$ and
$\phi_{ik}$ simultaneously. Suppose that $k$ entered last, then when $k$ arrives at $\phi_{ki}$,
$j$ is flying in $P_i$ and has reached $\phi_{ik}$ (by previous conditions). As a consequence,  $k$
does not enter in $P_i$ because it detects that $P_i$ is occupied by $j$, a contradiction.
\end{proof}



\begin{figure}[h]
\centering
\includegraphics[scale=0.6, trim = 2cm 1cm 0cm 0cm, clip=true]{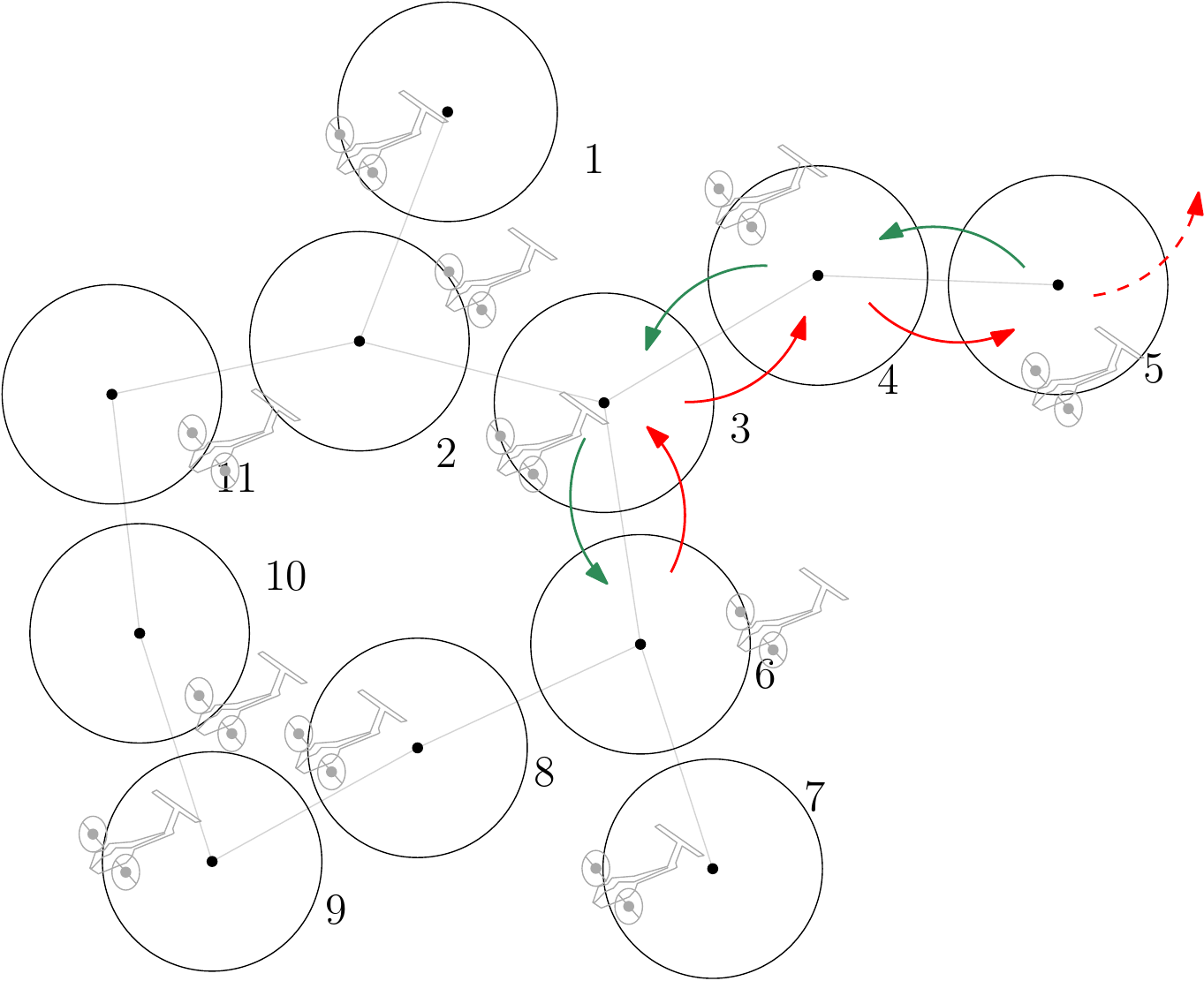}
    \caption{The AR $6$ can leave the system swapping with $3$, $4$ and $5$ in this order.}
    \label{fig:swappingToOut}
\end{figure}

\begin{figure}[h]
\centering
\begin{subfigure}{0.5\textwidth}
\centering
\includegraphics[width=.4\textwidth , page=2]{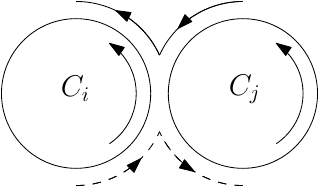}
    \caption{One direction}
    \label{fig:swappingSameDir}
\end{subfigure}%
\begin{subfigure}{0.5\textwidth}
\centering
\includegraphics[width=.4\textwidth , page=3]{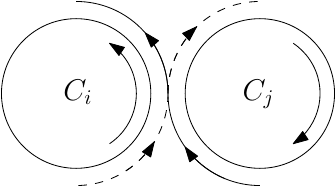}
    \caption{Two directions}
\label{fig:swappingDiffDir}
\end{subfigure}
\caption{Swapping by using one or two directions. The swapping paths in (b) is smoother than in (a).}
\end{figure}


\section{The synchronization algorithm}\label{sec:synchro_alg}
The synchronization algorithm consists of several stages, described below.
\subsection{Computing the communication graph}

Notice that, in the simplified model, the communication graph is actually the \emph{intersection graph} of a set of enlarged trajectories, namely, the set of disks of radius $1+r/2$ with the same trajectory centers. Therefore, the intersection graph can be computed in linear time (\cite{bentley}).

\subsection{Computing the maximum bipartite subgraph}

We can test the bipartiness of the communication graph in linear time. If it is bipartite we are done,
else a reasonable strategy is to use a bipartite subgraph with the maximum number of edges.
Finding the maximum bipartite subgraph of a topological graph
is an NP-complete problem (\cite{karpr72}). However, such a subgraph can be found in polynomial
time when the input graph is planar (\cite{Hadlock75}). Alternatively, there exist various
approximation algorithms that can be useful in practice (see for instance \cite{eppstein94}).
%
%
%
%
%

It should be noted that the condition of planarity is not overly restrictive.
Many applications produce planar communication graphs.
Indeed, robot trajectories can be changed (and even shortened) to avoid double crossings.
Moreover, we can assume that robots cannot turn at sharp angles. As a consequence,  the communication
range $r$ is typically much smaller that the ``width'' of the curve and no crossing occurs.

 \subsection{ Scheduling the flights with opposite directions.}
Algorithm~\ref{alg:preprocessing}, shown below, preprocesses the trajectories and
computes the starting positions of all ARs in the team. The input values are: the list of the pairwise disjoint unit circles ($C$), the communication range of the ARs ($r$) and the tour time ($T$), and the output values are: the list of starting positions of the ARs (\emph{SPos}), the list of the (initial) directions of movement of the ARs ($dirs$) and a structure $M$ that contains synchronization information ($M[i,j]$ stores the region of communication between the $C[i]$ and $C[j]$ if it exists and it is used in the synchronization).
%
%

\begin{algorithm}[h!]
\caption{Preprocessing to deploy ARs team}
\label{alg:preprocessing}
  \SetKwInOut{Input}{input}
  \SetKwInOut{Output}{output}
  \BlankLine
  \Input{$C, r, T$}
  \Output{$SPos,dirs, M$}
  \BlankLine

  $G\gets underGraph(C,r)$\;
  \If{\NOT $G$ is bipartite}{
  	$G\gets maxBiSubgraph(G)$\;
  }
  \If{\NOT $G$ is synchronizable}{
 	$G\gets maxSynchSubgraph(G)$\;
  }
  $v\gets first(V(G))$\;
  $q\gets queue()$\;
  $q.enqueue(v)$\;
  $analyzed\gets \{\}$\;
  $SPos\gets list[V(G)]$\;
  $dirs\gets list[V(G)]$\;
  $SPos[v]\gets 0$\;
  $dirs[v]\gets true$\;
  \While{\NOT $q$ is empty}{
  	$w\gets q.dequeue()$\;
  	$m\gets G.getNeighbors(w)$\;
  	\ForAll{$a\in m$}{
  		\If{$a\not\in analyzed$}{
  			$analyzed.add(a)$\;
  			$\beta_{w,v}\gets angleOf(<w,v>)$\;
  			$SPos[a]\gets 2*\beta_{w,v}-SPos[w]-\pi$\;
  			$dirs[a]\gets \NOT dirs[w]$\;
  			$q.enqueue(a)$\;
  		}
  	}
  }
  $M\gets encode(G)$\;
  \Return $SPos,dirs, M$\;
\end{algorithm}

\begin{algorithm}[h!]
	\caption{Onboard execution}
	\label{alg:onboard}
	\SetKwInOut{Input}{input}
	\Input{$C,i,M,\alpha,d,t,wPlan$}

	\BlankLine
	\quad fly to assigned trajectory $C[i]$\;
	\quad locate in position $\alpha$\;
	\quad wait until $cTime=t$\;
	\quad $wA\gets \{C[i]\}$\;
	\quad $fPlan\gets flightPlan(M,wPlan,t,\alpha,d,wA)$\;
	\quad $load(fPlan)$\;
	\BlankLine
	\While{\NOT ABORT}{
		$doWork(wPlan, cTime, cPos)$\;
		\If{entering in communication region}{
			$openConnections()$\;
			$f\gets false$\;
		}
		\If{within communication region}{
			$ni\gets getNeighborIndex(M,cPos)$\;
			\If{detect a neighbor}{
				$f\gets true$\;
				$shareInfo()$\;
				\If{$C[ni]\in wA$}{
					$wA\gets split(wA, C[ni])$\;
					$fPlan\gets flightPlan(M,wPlan,$
					
					\quad\quad\quad\quad\quad\quad$cTime,cPos, d,wA)$\;
					$load(fPlan)$\;
				}
			}
		}
		\If{coming out communication region}{
			$closeConnections()$\;
			\If{$\NOT f \AND C[ni]\not\in wA$}{
				$wA\gets join(wA, S[ni])$\;
				$fPlan\gets flightPlan(M,wPlan,$

				\quad\quad\quad\quad\quad\quad$cTime,cPos,d,wA)$\;
				$load(fPlan)$\;
			}
		}
	}

\end{algorithm}


Algorithm~\ref{alg:onboard} runs onboard of the ARs and takes as input the list of the disjoint unit circles ($C$), index of the assigned trajectory ($i$), the structure $M$ produced by Algorithm~\ref{alg:preprocessing}, the initial position ($\alpha=\mbox{\em SPos}[i]$), the directions of movement ($d=\mbox{\em dirs}[i]$), the performance starting time ($t$) and an abstract data structure with the info about the general task an local subtasks to perform (\emph{wPlan}).
The algorithm also uses auxiliary subroutines: \emph{flightPlan} (compute a flight plan), \emph{load} (load the flight plan in the navigation system), \emph{doWork} (do the abstract task), \emph{openConnections} (open the interface of communication), \emph{closeConnections} (close the interface of communication), \emph{shareInfo} (share stored information with the corresponding neighbor) and \emph{getNeighborIndex} (obtain index in $S$ of a neighbor trajectory).

%
%
%
%
%

\section{Generalization to non-circular trajectories and heterogeneous robots}\label{gener}
So far we have assumed homogeneous robots, with equal speeds and fuel capacities.
In this section we describe how to extend the results under a more general and more realistic setting.
Let us consider a team of $n$ ARs with different capabilities and performing their corresponding tasks through arbitrary disjoint closed trajectories $\{P_1,\dots,P_n\}$.

Naturally, we have some control on the speeds of the ARs as they can accelerate or decelerate to increase or decrease their speed. The speed can even change on different sections of the same trajectory.
The key to generalizing the synchronization scheme used in the simple theoretical
model is to force all members of the team to take (approximately) the same time to
make a tour of their respective trajectories. In the ideal theoretical model two
neighboring ARs reach the communication link at the same time. In a real
implementation of these strategies, however, we need to allow for some margin of error
in order to guarantee robustness. This can be accomplished by having each AR compare
its current and target locations at regular time intervals, and adjust its speed, if necessary,
in order to resynchronize.
We refer to the time that an AR takes to make a tour as the \emph{system period}.

We can use the communication links to partition a trajectory $P_i$ into sections and assign each of them a travel time $t_j$ such that $\sum_jt_j=T$, from this assignment we can compute the required speeds in each section. Once we know the required speed for each section, it suffices to know the starting position and movement direction of an AR in order to compute its position at any time. So, the problem in the general case is: given a system period $T$, partition the trajectories into sections, assign a time to each section so that the cumulative time of all sections is $T$, and compute the initial position for every AR in its trajectory such that the system is synchronized.

The construction of the communication graph is similar to the case of the simplified
model. We include a node in the communication graph for each trajectory $P_i$ and
include an edge between two nodes if the minimum distance between the respective
trajectories is within the communication range. (Note that we allow at most one edge
between two nodes. If there are multiple pairs of points on two trajectories whose
distances are within the communication threshold we choose the closest pair.)
To compute the edges we can use known algorithms from computational geometry (e.g.,
\cite{wang86,quinlan94}) to efficiently find the  minimum distance between two
polygons.
Then, let $S=\{P_1, P_2, \dots, P_n\}$ be a set of disjoint closed
trajectories assigned to a team of $n$ ARs. Also, let $G=(V,E)$ be the computed
communication graph on $S$ using the communication range $r_i$ of each AR. Borrowing
notation from previous sections, we use $\phi_{ij}$ to denote the location of $P_i$
closest to $P_j$, for $(i,j)\in E$. (Note that in the generalization $\phi_{ij}$ is
not an angle as in circular model, but rather a location in the curve $P_i$ that can
be parameterized to a value in $[0,2\pi)$ if so desired).

Let $T$ denote the system period. If the communication graph is a tree, then considering each trajectory as consisting of a single section, we obtain that the speed in each trajectory $P_i$ is constant and equal to the ratio $l_i/T$ where $l_i$ is the length of $P_i$. Then fixing arbitrary the initial position of a AR $i$, we can compute the initial position of every neighbor of $i$ as follows. Let $j$ be a neighbor of $i$ and suppose $i$ takes $t$ units of time to reach $\phi_{ij}$ from its starting position. Thus $i$ and $j$ are synchronized if  $j$ takes $t$ units of time to reach $\phi_{ji}$ from its starting position (see Figure~\ref{fig:two_gen_synchro}).
Note that the computed initial position for $j$ depends on its chosen direction of movement. It is easy to see that  in this way we can compute the initial positions for all ARs in the synchronized system.
If the communication graph contains cycles, this simple approach does not work because considering each trajectory as a single section forces the ARs to use constant speed (possibly different from each other)
\footnote{However, it easy to see that the same approach that worked for unit circles would work here if we parameterize position to a value in $[0,2\pi)$ and allow the ARs to change its speed to cover different distances in physical space but equal distance in parameter space.}.

\begin{figure}
\centering
\includegraphics[page=2]{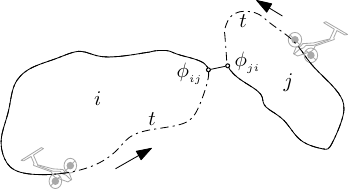}
\caption{The dashed-dotted stroke is the section of trajectory to take to achieve $\phi_{ij}$ and $\phi_{ji}$ respectively.
Both should use the same amount of time ($t$) to achieve the link position.}
\label{fig:two_gen_synchro}
\end{figure}

Figure~\ref{fig:cycle_gen_synchro} shows a cycle in the communication graph, if we know the starting position of $1$ we can compute the starting position of $2$, then of $3$, and so on until we have computed the starting position of $k$ using the starting position of $k-1$. But having done this, ARs $1$ and $k$ are not necessarily synchronized. We describe conditions that help us determine when a cycle can be synchronized with ARs flying in the same direction or with neighbors in opposite directions.

First, we introduce notation needed for the rest of this section. Let $t_i$ (resp. $r_i$) be the time it takes the AR $i$ to travel the inside (resp. outside) section of its trajectory in the cycle (see Figure~\ref{fig:cycle_gen_synchro}). We use ellapsed time to describe the location of ARs as follows: $\phi_{ij}+t$ denotes the position of $i$ obtained by moving CCW from $\phi_{ij}$ during $t$ units of time. Analogously, $\phi_{ij}-t$ denotes the position of $i$ obtained by moving CW from $\phi_{ij}$ during $t$ units of time.

\begin{figure}[h]
\centering
\includegraphics[width=.4\textwidth, page=1]{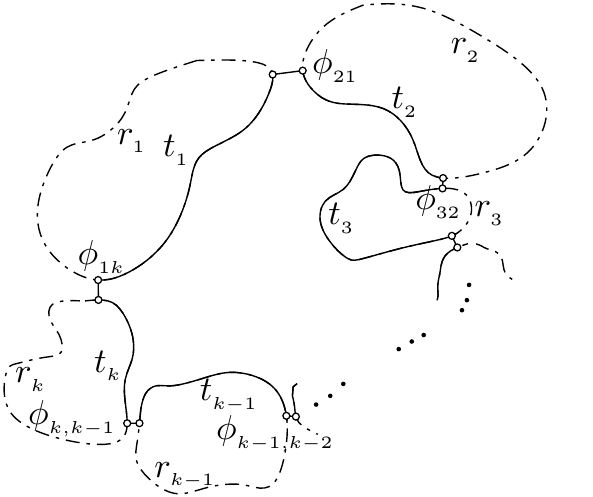}
\caption{ A cycle in the communication graph.}
\label{fig:cycle_gen_synchro}
\end{figure}

\begin{theorem}\label{thm:cycle_synchro_general1}
Let $P_1, P_2, \dots, P_k$ be $k$ trajectories that form a cycle in the communication graph. Let $T$ be the system period. If the ARs are flying in the same direction the cycle can be synchronized if and only if there exists $z\in\mathbb{N}$ such that:
\begin{align*}
t_1+t_2+\dots+t_k&=zT \text{\quad and}\\
r_1+r_2+\dots+r_k&=(k-z)T.
\end{align*}
\end{theorem}
\begin{proof}
In the equations below we use $\alpha_i$ to denote the starting position of AR $i$ in $P_i$.
Without loss of generality, suppose that the ARs fly CCW and that $\alpha_1=\phi_{1k}$. Computing $\alpha_2,\alpha_3,\dots,\alpha_k$ we obtain the following result:
\begin{align*}
\alpha_1 &= \phi_{1k}  \\
\alpha_2 &= \phi_{21}-t_1  \\
\alpha_3 &= \phi_{32}-t_1-t_2  \\
\vdotswithin{\alpha_{k}}& \vdotswithin{= \phi_{32}-t_1-t_2}  \\
\alpha_{k} &= \phi_{k,k-1}-t_1-t_2-\dots-t_{k-1}   
\end{align*}
Obviously, having that $i$ is synchronized with $i+1$ for $1\leq i <k$ then the cycle can be synchronized if and only if $\alpha_1 = \phi_{1k}-t_1-t_2-\dots-t_{k}=\phi_{1k}$, from here we deduce that $t_1+t_2+\dots+t_k$ is a multiple of $T$ because $\phi_{ij}=\phi_{ij}+T$. Note that $t_i<T$ for all $1\leq i\leq k$, then
\begin{equation}\label{eq:t_sum}
t_1+t_2+\dots+t_k=zT\quad z\in\mathbb{N}, (0<z<k)
\end{equation}
Since $t_i+r_i=T$ then
\begin{equation}\label{eq:all_sum}
\left(t_1+r_1\right)+\left(t_2+r_2\right)+\dots+\left(t_k+r_k\right)=kT
\end{equation}
The difference between (\ref{eq:all_sum}) and (\ref{eq:t_sum}) is
\begin{equation}
r_1+r_2+\dots+r_k=(k-z)T.
\end{equation}
%
\end{proof}

Note that $k$ is not necessarily an even number, but in the simplified model $t_1+t_2+\dots+t_k$ is a multiple of $T$ only if $k$ is even. This is consistent with the bipartition requirement.

Let us now consider the case of neighbors flying in opposite directions. In this case, the communication graph must be bipartite.

\begin{theorem}\label{thm:cycle_synchro_general2}
Let $P_1, P_2, \dots, P_{2k}$ be $2k$ trajectories that form a cycle in the communication graph. Let $T$ be the system period. If the neighbors are flying in opposite directions the cycle can be synchronized if and only if exists $z\in\mathbb{N}$ such that:
\begin{align*}
 t_1+r_2+t_3+\dots+t_{2k-1}+r_{2k}&=zT \text{\quad and}\\
 r_1+t_2+r_3+\dots+r_{2k-1}+t_{2k}&=(2k-z)T
\end{align*}

\end{theorem}
\begin{proof}
In the equations we use $\alpha_i$ to denote the starting position of AR $i$ in its trajectory.
Without loss of generality, suppose that $\alpha_1=\phi_{1k}$ and that AR $1$ is flying CCW. Computing $\alpha_2,\alpha_3,\dots,\alpha_k$ we obtain the following result:
\begin{align}
\alpha_1 &= \phi_{1k} \nonumber \\
\alpha_2 &= \phi_{21}+t_1 \nonumber \\
\alpha_3 &= \phi_{32}-t_1-r_2 \nonumber \\
\vdotswithin{\alpha_{2k-1}}& \vdotswithin{= \phi_{32}-t_1-r_2} \nonumber \\
\alpha_{2k-1} &= \phi_{2k-1,2k-2}-t_1-r_2-\dots-r_{2k-2} \nonumber \\
\alpha_{2k} &= \phi_{2k,2k-1}+t_1+r_2+\dots+r_{2k-2}+t_{2k-1} \nonumber 
\end{align}
Obviously, having that $i$ is synchronized with $i+1$ for $1\leq i <2k$ implies that the cycle can be synchronized if and only if $\alpha_1 = \phi_{1k}-t_1-r_2-\dots-t_{2k-1}-r_{2k}=\phi_{1k}$, from here we conclude that $t_1+r_2+\dots+t_{2k-1}+r_{2k}$ is a multiple of $T$ because $\phi_{ij}=\phi_{ij}+T$. Note that $t_i<T$ for all $1\leq i\leq 2k$, then
\begin{equation}\label{eq:t_sum_op}
t_1+r_2+\dots+t_{2k-1}+r_{2k}=zT\quad z\in\mathbb{N}, (0<z<2k)
\end{equation}
Since $t_i+r_i=T$ then
\begin{equation}\label{eq:all_sum_op}
\left(t_1+r_1\right)+\left(t_2+r_2\right)+\dots+\left(t_{2k}+r_{2k}\right)=2kT
\end{equation}
The difference between (\ref{eq:all_sum_op}) and (\ref{eq:t_sum_op}) is
\begin{equation}
r_1+t_2+\dots+r_{2k-1}+t_{2k}=(2k-z)T.
\end{equation}

%
An analogous result is obtained if AR $1$ is flying CW.
\end{proof}


\subsection{Application to a case study}
In this subsection we present an example of how to apply our strategy (using opposite
directions) to a communication graph that is neither a tree nor an odd cycle. By
decomposing the communication graph into a collection of cycles and trees, the
approach used here can be similarly adapted to other complex configurations.
Figure~\ref{fig:video_sample} shows the trajectories and the chosen edges of our case
study. Notice that we have irregular trajectories of different lengths, so, using the
same speed for all the ARs is not possible for a fully synchronized system. We may
consider assigning different and constant speeds for each trajectory such that all the
ARs take the same time in a tour, but this does not guarantee synchronization either
because the communication graph has two cycles, $2,7,8,5$ and $2,5,8,6,4,3$. We can
solve the problem by a careful application of
Theorem~\ref{thm:cycle_synchro_general2}. This requires that we extend our notation to
cover trajectory sections.

\begin{figure}[h]
\centering
\begin{subfigure}[b]{.4\textwidth}
\centering
\includegraphics[scale=.33]{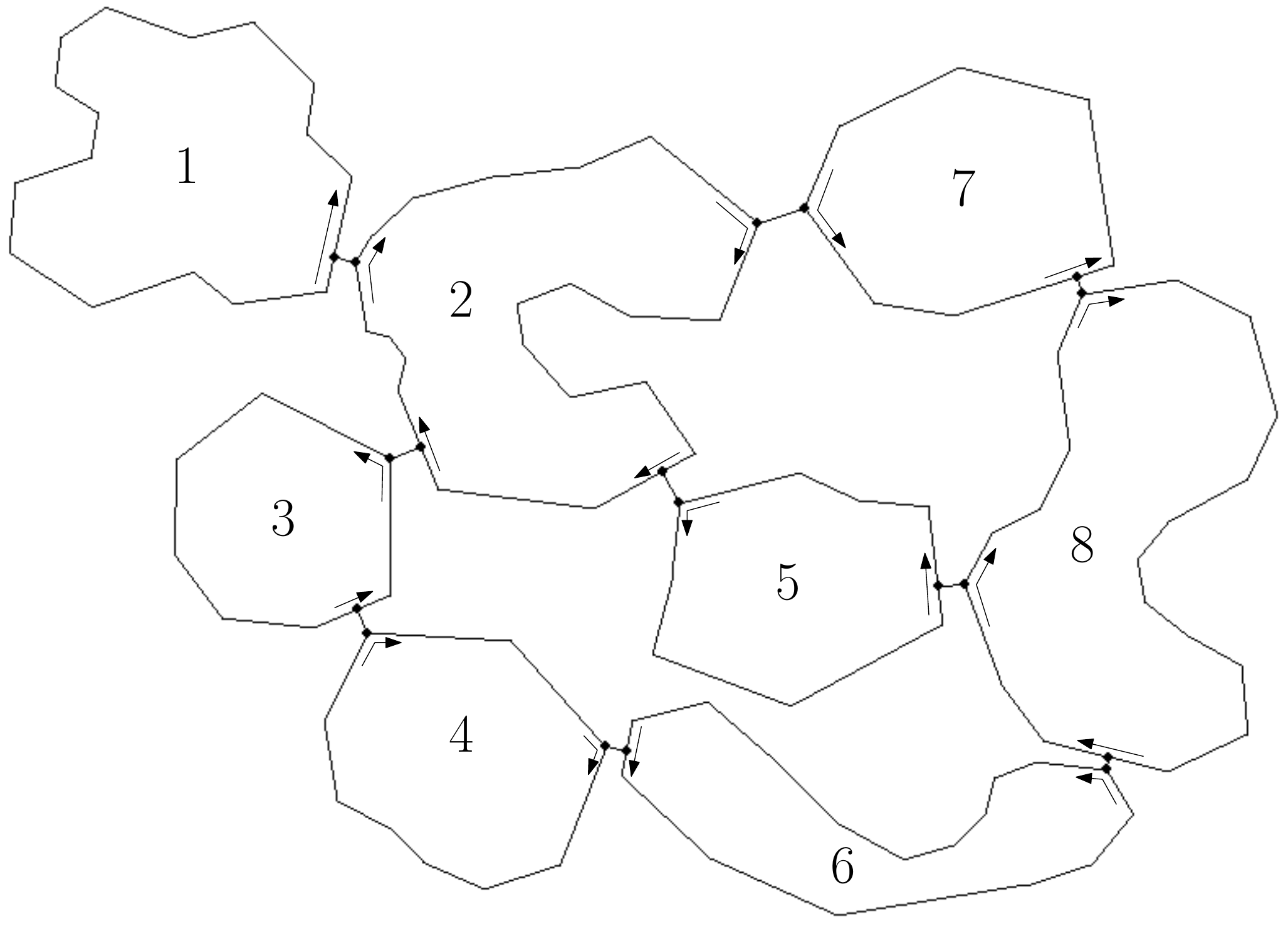}
\caption{}
\label{fig:video_sample}
\end{subfigure}\qquad\qquad%
\begin{subfigure}[b]{.45\textwidth}
\centering
\includegraphics[scale=.45, page=2]{general_sample.pdf}
\caption{}
\label{fig:section_times}
\end{subfigure}
\caption{(a) Example of a general scenario. (b) Time spent while traveling various trajectory sections.}
\end{figure}

We denote by $\ptime{i}{j}{l}$ the time spent to travel in trajectory $i$ from the
link position shared with trajectory $j$ to the link position shared with trajectory
$l$ (following the assigned direction), see the Figure~\ref{fig:section_times}. Note
that $\ptime{i}{j}{l}$ can be different from $\ptime{i}{l}{j}$, and
$\ptime{i}{j}{l}+\ptime{i}{l}{j}=T$ where $T$ is the period (also
$\ptime{i}{j}{l}+\ptime{i}{l}{k}=\ptime{i}{j}{k}$).

Applying Theorem~\ref{thm:cycle_synchro_general2} to cycle $2,7,8,5$ we obtain the following equations:
\begin{align}
\ptime{2}{7}{5}+\ptime{7}{8}{2}+\ptime{8}{5}{7}+\ptime{5}{2}{8}&=z_1T\label{eq:eq1}\\
{\left(\ptime{2}{5}{3}+\ptime{2}{3}{7}\right)}+\ptime{7}{2}{8}+{\left(\ptime{8}{7}{6}+\ptime{8}{6}{5}\right)}+\ptime{5}{8}{2}&=(4-z_1)T\label{eq:eq2}
\end{align}
with $z_1\in\mathbb{N},\;0<z_1<4$.

From cycle $2,5,8,6,4,3$ we obtain:
\begin{align}
\ptime{2}{5}{3}+\ptime{5}{8}{2}+\ptime{8}{6}{5}+\ptime{6}{4}{8}+\ptime{4}{3}{6}+\ptime{3}{2}{4}&=z_2T \label{eq:eq3}\\
\begin{array}{r}
{\left(\ptime{2}{3}{7}+\ptime{2}{7}{5}\right)}+\ptime{5}{2}{8}+{\left(\ptime{8}{5}{7}+\ptime{8}{7}{6}\right)}\\+\ptime{6}{8}{4}+\ptime{4}{6}{3}+\ptime{3}{4}{2}
\end{array}&=(6-z_2)T\label{eq:eq4}
\end{align}
with $z_2\in\mathbb{N},\;0<z_2<6$.

Also we have that:
\begin{equation}
\begin{array}{rcr}
\ptime{2}{7}{5}+\ptime{2}{5}{3}+\ptime{2}{3}{7}=T &\qquad& \ptime{5}{2}{8}+\ptime{5}{8}{2}=T\\
\ptime{8}{5}{7}+\ptime{8}{7}{6}+\ptime{8}{6}{5}=T &\qquad& \ptime{5}{8}{2}+\ptime{5}{2}{8}=T\\
\ptime{3}{2}{4}+\ptime{3}{4}{2}=T &\qquad& \ptime{4}{3}{6}+\ptime{4}{6}{3}=T\\
\ptime{6}{8}{4}+\ptime{6}{4}{8}=T&\qquad&
\end{array}\label{eq:period_constraint}
\end{equation}

We proceed by assigning values to each $\ptime{i}{j}{k}$ fulfilling the constraints
(\ref{eq:eq1}), (\ref{eq:eq2}), (\ref{eq:eq3}), (\ref{eq:eq4}) and (\ref{eq:period_constraint}).
Also, we must try to assign time values so as to keep small differences in speed for
various trajectory sections.
By doing this, 
we reduce the accelerations and decelerations of the ARs in the mission and keep
flight behavior realistic. If we do not keep this in mind while working with a group of
heterogeneous ARs, it may happen that some time during the mission an AR must travel a
subtrajectory with a speed that can not be reached.

Using a heuristic algorithm we obtain the following feasible assignation:
$$\begin{array}{c}
\ptime{2}{7}{5}=0.40T\qquad\ptime{2}{5}{3}=0.18T\qquad\ptime{2}{3}{7}=0.42T\\
\ptime{8}{5}{7}=0.28T\qquad\ptime{8}{7}{6}=0.50T\qquad\ptime{8}{6}{5}=0.22T\\
\ptime{3}{4}{2}=0.40T\qquad\ptime{3}{2}{4}=0.60T\\
\ptime{4}{3}{6}=0.30T\qquad\ptime{4}{6}{3}=0.70T\\
\ptime{5}{8}{2}=0.34T\qquad\ptime{5}{2}{8}=0.66T\\
\ptime{6}{8}{4}=0.64T\qquad\ptime{6}{4}{8}=0.36T\\
\ptime{7}{2}{8}=0.34T\qquad\ptime{7}{8}{2}=0.66T
\end{array}$$
After that, fixing a value $T$ such that every AR can reach the associated speed to each subtrajectory (length of the subtrajectory$/$assigned time to the subtrajectory). Having the times (and the speeds) at each subtrajectory, setting an initial position for an AR, we can constructively compute the initial position for every AR in the communication graph.

We implemented a simulation of an example that leads us to the communication graph of the Figure~\ref{fig:video_sample}, a video is available at {\small\url{https://www.youtube.com/watch?v=T0V6tO80HOI}} illustrating all the phases of the algorithm.


\section{Simulation and computational results}\label{simul}
Consider the cooperative surveillance of an area by means of a team of small fixed wing aerial robots. The area can be divided into a grid with cells of 300m$\times$300m. Each cell is assigned to a member of the team. The robots are equipped with an on-board camera Panasonic DMC-GH2. The chosen focal length of the camera is 14mm, thus the field of view is $53.13^{\circ} \times 36.87^{\circ}$. The robots are programmed to fly at a constant altitude of 90m (to prevent the obstruction of commercial air traffic) and at a constant speed of 12m/s. From this altitude the covered area by the camera is approximately 45m$\times$30m and objects are distinguishable over 0.06m 
(the targets to detect should be grater than 0.06m). To cover a 300m$\times$300m cell we can use a \emph{back and forth} closed trajectory as shown in Figure \ref{fig:covering_area}, note that the farthest points to the trajectory are at a distance of 20.5m, so in these critic points we have a margin of deviation on the route of $\pm 2$m. Thus, every point can be watched by the camera in some instant during the tour. (A\~nadir aqu\'i, si consideran necesario, algo sobre ``kinematic contraints''. En esta trayectoria todos los cambios de direcci\'on est\'an sobre arcos de circunferencia de al menos 16.67m de radio como muestra la figura, por tanto los UAV deben ser capaces de girar sobre un c\'irculo de 16.67m de radio a 12m/s para que sea una trayectoria realmente factible)

\begin{figure}[!h]
\centering
\includegraphics[page=3, width=\textwidth]{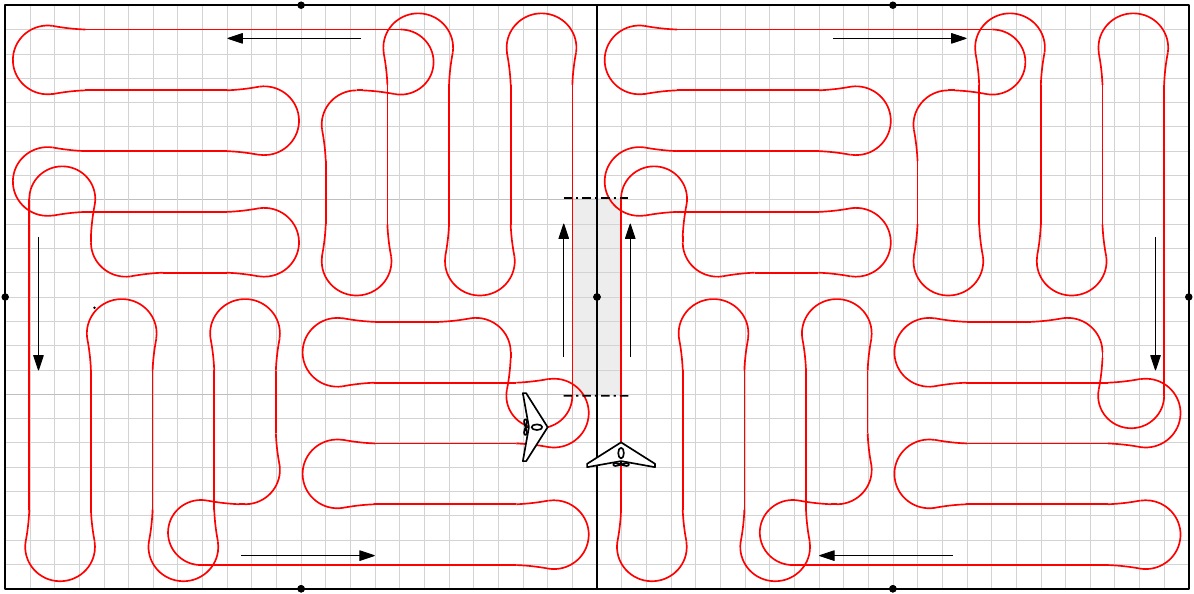}
\caption{Two neighboring cells and two synchronized UAVs. Using red stroke feasible trajectories to cover the cells. The UAV in the left cell is flying in CCW direction while the UAV in the right cell is flying in CW direction. The communication region between them is shaded in gray and has a length of 100m. The UAVs have been magnified to make them look better.}
\label{fig:covering_area}
\end{figure}

The robots have very constrained communication due to security reasons. Then, the only possibility to communicate the information is to synchronize the robots in such a way that two robots meet in adjacent cells. The objective is to synchronize the system in such a way that all the robots transmit the collected information in spite of robot failures and leaving for refueling as mentioned in Section \ref{sec:problem}. In our scenario the link position between two adjacent trajectories is at the middle point of the common side between the corresponding cells. The distance between two neighboring trajectories in the link position is 25m, close enough to share information using the on-board wifi equipment and sufficiently far to avoid collisions, see Figure \ref{fig:covering_area}. The proposed trajectory is symmetrical, so an aerial robot makes the same route traveling from any link position to the next link position. Note that a grid using these trajectories in the cells is equivalent to a grid using circular trajectories, see Figure \ref{fig:equiv_our_circular}. Therefore, if a team of UAVs is monitoring a region divided into a grid, one robot per cell, such that every pair of neighbors is flying in opposite directions at the same constant altitude and the same constant speed using our back and forth trajectories, then, using by Theorem \ref{thm:opposite_cycles} the system can be synchronized. The communication region between two neighboring UAVs has a length of 100m (see Figure \ref{fig:covering_area}), so flying at 12m/s they have 8.33s to share information with relative speed between them of 0m/s. The Figure \ref{fig:surveillance_illustration} illustrates this scenario. 

\begin{figure}[!h]
\centering
\begin{subfigure}{\textwidth/2}
\centering
\includegraphics[page=5, scale=.4]{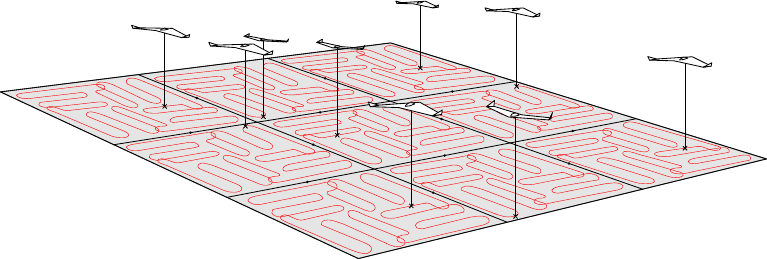}
\caption{}
\end{subfigure}%
\begin{subfigure}{\textwidth/2}
\centering 
\includegraphics[page=6, scale=.4]{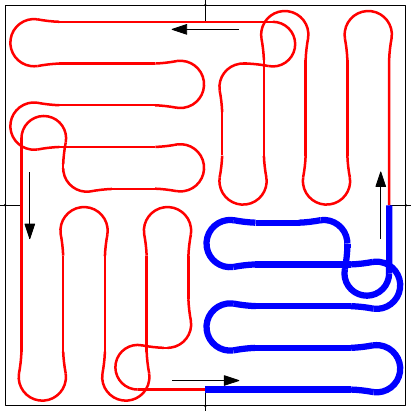}
\caption{}
\end{subfigure}
\caption{Equivalence between a grid using our back and forth trajectories and a grid using circular trajectories. Note that a segment of trajectory between two consecutive link points in (a) corresponds to a segment between two consecutive points in (b). }
\label{fig:equiv_our_circular}
\end{figure}

\begin{figure}[!h]
\centering
\includegraphics[page=4, width=\textwidth]{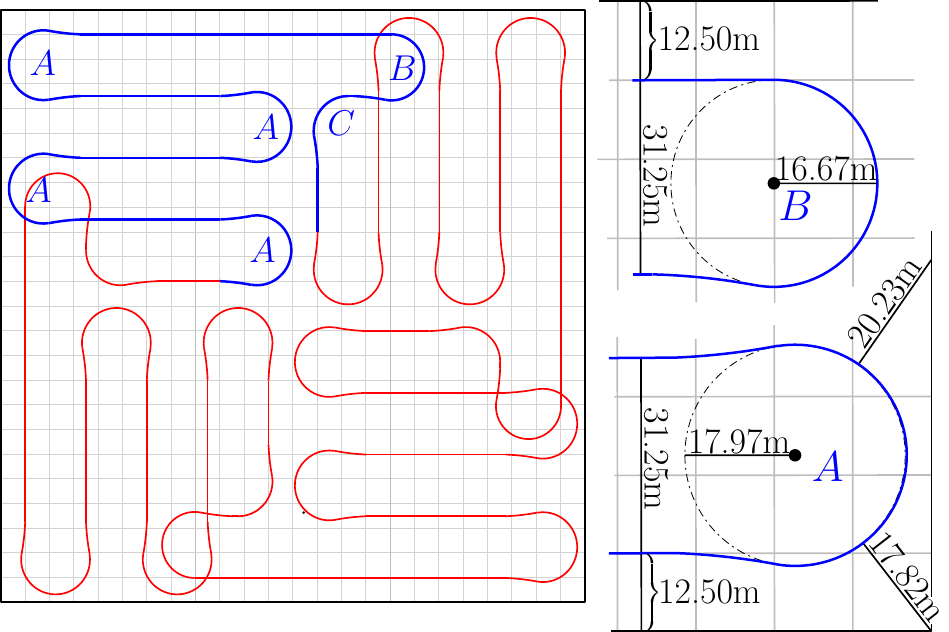}
\caption{A synchronized team of UAVs patrolling a $3\times 3$ grid region where the dimensions of the cells are 300m$\times$300m. Each UAV is flying in its cell over a trajectory as shown in the Figure \ref{fig:covering_area}. They are flying at the same constant altitude of 90m with constant speed of 12m/s. The UAVs have been magnified to make them look better.}
\label{fig:surveillance_illustration}
\end{figure}

The length of the proposed covering trajectories is approximately 3609.79m, thus an UAV spent 300.82s to make a tour in its cell. Simulations were performed to validate the proposed approach. We used the model with circular trajectories and assuming that the period of an AR is 300s in order to simplify the generation of test cases. We ran each simulation for 15\,000s.
We considered two measures to evaluate the performance of the system: the \emph{broadcast time (BT)}, that is, the time it takes for a message issued by a robot to reach the full team, and the \emph{abandoned time (AT)}, defined as the interval of time in which a trajectory is not attended by any AR.
Table~\ref{tab:simula_results} shows the results of the experiments. The first three cases correspond to communication graphs with a grid structure of size 3$\times$3, the following three correspond to a grid structure of size 5$\times$3 and The other cases correspond to graphs randomly generated.
%
%
%
A random graph is generated constructively, suppose that we have a connected graph formed by $m$ circles, to increase the graph we select a random circle $c$ and a random ray $r$ from the center of $c$, then the new circle is placed with center on $r$ such that it is disjoint with the other circles and keeping the connection in the graph. Thus, from a graph of $m$ nodes we have constructed a graph with $m+1$ nodes. With this idea starting with one circle we can construct graphs of any numbers of nodes.
The first column in the Table~\ref{tab:simula_results} shows the initial number of ARs; the second column, the number of fallen ARs (selected randomly); the third, the broadcast time; and the last shows the maximum abandoned times during the simulation. All the results correspond to the average of 10 simulations with the same parameters.
The code of the simulation was written using the \emph{Python} language.

\begin{table}[!h]
\centering
\begin{tabular}{c|ccrr}
&$\mathbf{N^\circ\;ARs}$&$\mathbf{N^\circ\;F.\;ARs}$&\textbf{Avg. BT(s)}&\textbf{Max. AT(s)}\\
\hline
\textbf{Grid $3\times 3$} & 9 & 0 & 348.71 & 0.00\\
\textbf{Grid $3\times 3$} & 9 & 2 & 394.46 & 157.87\\
\textbf{Grid $3\times 3$} & 9 & 4 & 482.33 & 277.50\\
\hline
\textbf{Grid $5\times 3$} & 15 & 0 & 543.30 & 0.00\\
\textbf{Grid $5\times 3$} & 15 & 3 & 595.50 & 260.63\\
\textbf{Grid $5\times 3$} & 15 & 7 & 658.43 & 607.50\\
\hline
\textbf{Random} & 10 & 0 & 786.34 & 0.00\\
\textbf{Random} & 10 & 2 & 812.48 & 327.00\\
\textbf{Random} & 10 & 5 & 1027.20 & 777.75\\
\hline
\textbf{Random} & 15 & 0 & 890.63 & 0.00\\
\textbf{Random} & 15 & 3 & 1099.16 & 285.75\\
\textbf{Random} & 15 & 7 & 1306.05 & 869.63\\
\end{tabular}

\caption{Simulations results using a circular model.}
\label{tab:simula_results}
\end{table}

Table~\ref{tab:simula_results} shows that our method is robust. The broadcast time does not grow significantly even if close to 50\% of the robots fail. Thus, the system is fault-tolerant. Also, it is worth noting the importance of a large number of communication links in the graph. In grid configurations, all the cycles meet the hypothesis of Theorem~\ref{thm:opposite_cycles} and we can use all the links in the communication graph. However, in random graphs the probability to generate cycles fulfilling such hypothesis is low. In this case, the algorithm computes the feasible maximal subgraph losing many communication links.
For random graphs, the abandoned time measure shows a decrease in performance. Finding strategies to improve the results gives us a promising line of research.
From here the importance to continue this research line to find strategies to tackle these cases.

\section{Starvation avoidance}\label{sec:starvation}
In this section we introduce a new concept corresponding to a phenomenon that may arise when one or more robots leave the system. Consider the communication graphs in Figure~\ref{fig:starvation_samples}. If the white ARs fail or otherwise leave the system, the two surviving ARs (shown in solid) fail to meet one another resulting in a permanent loss of synchronization. We say that an AR \emph{starves} if independent of how much longer it remains in flight, it permanently fails to encounter other ARs at any communication link. If some ARs leave the system and the remaining ARs are starving we say that the system \emph{falls into starvation}.

\begin{figure}[ht]
\centering
\begin{subfigure}{\columnwidth/2}
\centering
\includegraphics[scale=1.2]{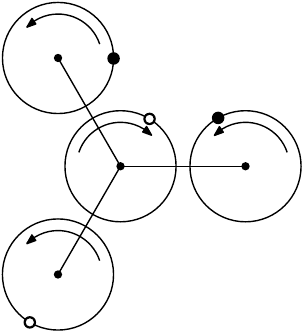}

\caption{}
\label{fig:starv_samp1}
\end{subfigure}%
\begin{subfigure}{\columnwidth/2}
\centering
\includegraphics[scale=1.2,page=2]{starvation.pdf}

\caption{}
\label{fig:starv_samp2}
\end{subfigure}%

    \caption{Examples of systems in starvation. If the white robots fail, the solid ones permanently fail to synchronize.}
        \label{fig:starvation_samples}
\end{figure}
Another example is illustrated in Figure~\ref{fig:starve_abandoned_trajectory}. By removing the white ARs the system falls into starvation and, if later,  ARs $a$ and $b$ also fail, then the trajectories $P_1$, $P_2$ and $P_3$ are abandoned forever.

\begin{figure}
\centering
\includegraphics[scale=1.2, page=4]{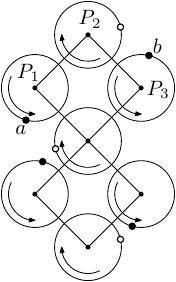}

\caption{Starvation and abandoned trajectories. The system begins to move at the drawn points. }
\label{fig:starve_abandoned_trajectory}
\end{figure}

Let $G$ be a (bipartite) communication graph of size $n$. For the sake of robustness, we are
interested in preserving the following properties in the presence of AR failures:
(i) all trajectories are covered
(ii) a broadcast sent from any AR reaches all surviving ARs (i.e., there are no starving ARs).

One idea to break starvation is to abandon the deterministic policy of switching to
a neighboring trajectory whenever the AR responsible for that trajectory is not
found. Instead, the AR decides with probability $p$ to switch to
the neighboring trajectory or not. Another strategy is to only allow migration
on edges of a subgraph of the communication graph, ideally one that is ``resistant" to starvation. An ideal choice is a Hamiltonian trajectory (at least when it can be found) as starvation is not possible in this case. In other words, no ARs can possibly starve in a chain provided there are at least two active ARs. A different topology with good properties is that of a spanning tree. For instance, it is easy to see that by permanently anchoring an AR in a leaf node one can guarantee that there are no starving ARs. In general, a graph has multiple spanning trees. One that is easy to compute is the depth-first-search (DFS) tree.

\subsection{Testing strategies to avoid starvation}
In this subsection we present some computational results to compare the aforementioned strategies.
The period of an AR is 80s (this is the time spent by an UAV to make a tour in a circular trajectory of radius 150m flying at a speed of 11.78m/s) and we ran each simulation for 4\,000s.
We use the following strategies:
\begin{itemize}
\item \textbf{alw} (when an AR arrives to a link position and its neighbor is not there then AR switches to its neighbor's trajectory).
\item \textbf{rand} (when an AR arrives to a link position and its neighbor is not there then AR switches trajectories with probability 1/2).
\item \textbf{dfs} (when an AR arrives to a link position and its neighbor is not there then AR switches trajectories if this edge is in the DFS tree with root a predefined node in the graph. In the case of grids the resulting DFS from the top-left vertex is a Hamiltonian-trajectory).
\end{itemize}
The measures to compare the strategies are the following:
\begin{itemize}
\item \textbf{Max. ST:} maximum interval of time in which a robot is starving.
\item \textbf{Avg. CT:} average of completed tours on each trajectory, this is an indicator of the performance in each trajectory.
\item \textbf{Max. AT:} the maximum interval of time in which a trajectory is not attended by any AR.
\item \textbf{Avg. BT:} average of the broadcast time.
\end{itemize}

We start testing these strategies with grid graphs and random graphs removing randomly about half of the robots. Observe that in Figures~\ref{fig:starvation_samples} and \ref{fig:starve_abandoned_trajectory} this number of failures is needed in order to produce starvation.

\begin{table}
\centering
\begin{tabular}{r|rrrr}
			& \textbf{Max. ST(s)} & \textbf{Avg. CT} & \textbf{Max. AT(s)} & \textbf{Avg. BT(s)}\\
			\hline
\textbf{alw} 	& 115	& 7.44 	& 64 	& 183.16\\
\textbf{dfs}	& 195	& 9.44	& 224 	& 156.80\\
\textbf{rand} 	& 175	& 11.11	& 504	& 135.46\\

\end{tabular} \\
\caption{Grid $3\times 3$, removing 4 ARs randomly.} %
\label{tab:grid_3x3}
\end{table}
\begin{table}
\centering
\begin{tabular}{r|rrrr}
			& \textbf{Max. ST(s)} & \textbf{Avg. CT} & \textbf{Max. AT(s)} & \textbf{Avg. BT(s)}\\
									\hline
\textbf{alw} 	& 115	& 6.93 	& 124 	& 186.24\\
\textbf{dfs}	& 115	& 9.87	& 144 	& 178.98\\
\textbf{rand} 	& 234	& 11.13	& 724	& 216.41\\

\end{tabular}
\caption{Grid $5\times 3$, removing 7 ARs randomly.}
\label{tab:grid_5x3}
\end{table}

The results are shown in Tables \ref{tab:grid_3x3}, \ref{tab:grid_5x3} and \ref{tab:rand}. Note that we obtain the same results using \textbf{alw} and \textbf{dfs} in the statistics of the random graphs. The reason is that the bipartite subgraph fulfilling Theorem~\ref{thm:cycle_synchro_general2} is a tree, so, all the edges of the communication graph are in the DFS-tree. In neither of these cases the system falls into starvation (apparently, starvation is improbable when removing ARs randomly). Also, our initial strategy (\textbf{alw}) obtains better results than the others in these random cases, showing robustness (with respect to starvation) in most cases.

\begin{figure}[ht]

\begin{subfigure}{.5\textwidth}
\centering
\includegraphics[scale=0.33]{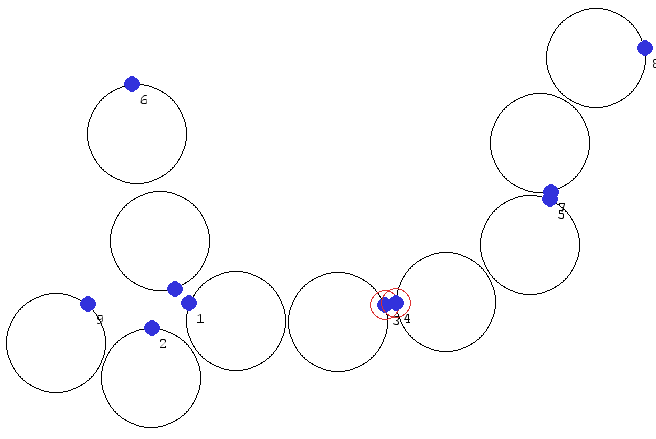}
\caption{}
\label{fig:random_10_1}
\end{subfigure} %
\begin{subfigure}{.5\textwidth}
\centering
\includegraphics[scale=0.33]{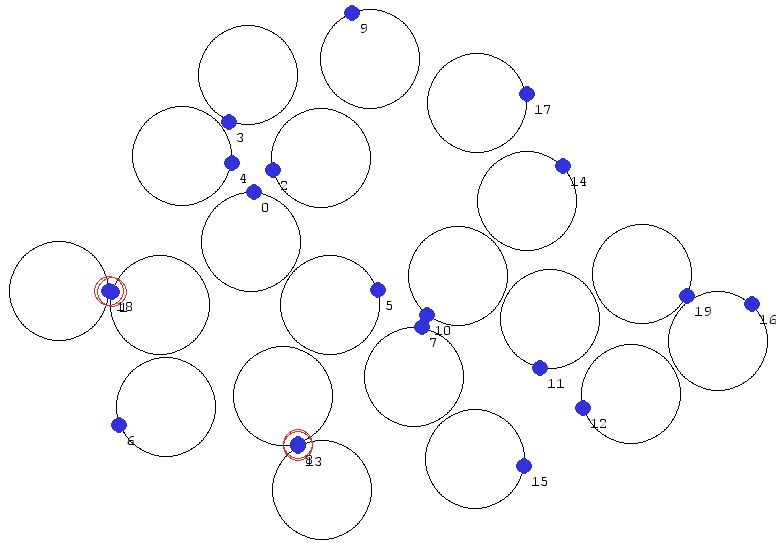}
\caption{}

\label{fig:random_20_2}
\end{subfigure}

\caption{Random graphs, the bipartite subgraphs fulfilling Theorem~\ref{thm:cycle_synchro_general2} are trees.}
\end{figure}

\begin{table}[h]
\centering
\begin{tabular}{r|rrrr}
			& \textbf{Max. ST(s)} & \textbf{Avg. CT} & \textbf{Max. AT(s)} & \textbf{Avg. BT(s)}\\
			\hline
\textbf{(alw/dfs)$^{1}$} 		& 224	& 11.10 	& 118 	& 238.36\\
\textbf{rand$^{1}$} 			& 603	& 10.10		& 1398	& 422.18\\
\hline
\textbf{(alw/dfs)$^{2}$}		& 535	& 9.35		& 239	& 518.20\\
\textbf{rand$^{2}$}			& 900	& 11.45		& 1144	& 789.82\\
\end{tabular}
\caption{$(1)$-Figure \ref{fig:random_10_1} removing 5 ARs randomly and $(2)$-Figure \ref{fig:random_20_2} removing 10 ARs randomly.}
\label{tab:rand}
\end{table}

Let us now consider other instances of starvation and discuss how the above strategies behave in these cases.

\begin{figure}[ht]
\centering
\begin{subfigure}{.5\textwidth}
\centering
\includegraphics[scale=1, page=3]{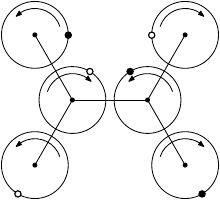}
\caption{}
\label{fig:starv_samp3}
\end{subfigure}%
\begin{subfigure}{.5\textwidth}
\centering
\includegraphics[scale=1, page=5]{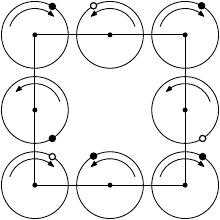}
\caption{}
\label{fig:starv_samp4}
\end{subfigure}
\caption{Starvation samples. If the white nodes are removed the system falls into starvation.}
\end{figure}

Figures \ref{fig:starv_samp3} and \ref{fig:starv_samp4} show two other cases when if the white ARs leave the system, then the team falls into starvation. Table~\ref{tab:starv_trees} shows the behavior of the strategies in the graphs of the Figures \ref{fig:starv_samp1}, \ref{fig:starv_samp2} and \ref{fig:starv_samp3} after the white ARs leave the system. The strategies \textbf{alw} and \textbf{dfs} have the same results because these graph are trees. In these cases the system falls into starvation. Using the probabilistic strategy, we obtain better results. Note that using the \textbf{alw} or \textbf{dfs} the trajectories are visited more frequently, but the broadcast time is infinity (that is, starvation occurs).

\begin{table}[ht]
\centering
\begin{tabular}{r|rrrr}
			& \textbf{Max. ST(s)} & \textbf{Avg. CT} & \textbf{Max. AT(s)} & \textbf{Avg. BT(s)}\\
\hline
\textbf{(alw/dfs)$^1$} 		& 4000	& 6.25 	& 64 	& $\infty$\\
\textbf{rand$^1$} 			& 422	& 12.50	& 571	& 97.44\\
\hline
\textbf{(alw/dfs)$^2$}		& 4000	& 8.00	& 144	& $\infty$\\
\textbf{rand$^2$}			& 535	& 13.60	& 1084	& 138.76\\
\hline
\textbf{(alw/dfs)$^3$}		& 4000	& 8.33	& 65	& $\infty$\\
\textbf{rand$^3$}			& 395	& 11.00 & 844	& 163.58
\end{tabular}
\caption{$(1)$-Figure \ref{fig:starv_samp1}, $(2)$-Figure \ref{fig:starv_samp2} and $(3)$-Figure \ref{fig:starv_samp3}}.
\label{tab:starv_trees}
\end{table}

The following tables show the results in the example graphs with cycles (Figures~\ref{fig:starve_abandoned_trajectory} and \ref{fig:starv_samp4}), and then using the \textbf{dfs} strategy we obtain some interesting results. The maximum starvation time is lower in both cases when using \textbf{dfs}. The number of completed tours is higher using \textbf{dfs} in the first case (Table~\ref{tab:starv_tab1}), while \textbf{rand} performs better in the second case (Table~\ref{tab:starv_tab2}). Finally, with respect to broadcast time, \textbf{dfs} shows better performance  in the second case, while \textbf{rand} is a better choice for the first case.

\begin{table}
\centering
\begin{tabular}{r|rrrr}
			& \textbf{Max. ST(s)} & \textbf{Avg. CT} & \textbf{Max. AT(s)} & \textbf{Avg. BT(s)}\\
\hline
\textbf{alw} 	& 4000	& 0.00 	& 24 	& $\infty$\\
\textbf{dfs}	& 155	& 10.14	& 144 	& 158.77\\
\textbf{rand} 	& 275	& 8.71	& 384	& 147.14\\
\end{tabular}
\caption{Results corresponding to Figure~\ref{fig:starve_abandoned_trajectory}.}
\label{tab:starv_tab1}
\end{table} %
\begin{table}
\centering
\begin{tabular}{r|rrrr}
			& \textbf{Max. ST(s)} & \textbf{Avg. CT} & \textbf{Max. AT(s)} & \textbf{Avg. BT(s)}\\
\hline
\textbf{alw} 	& 4000	& 0.00 	& 24 	& $\infty$\\
\textbf{dfs}	& 235	& 9.38	& 184 	& 177.23\\
\textbf{rand} 	& 415	& 13.38	& 564	& 223.90\\
\end{tabular}
\caption{Results corresponding to Figure~\ref{fig:starv_samp4}.}
\label{tab:starv_tab2}
\end{table}

Summarizing, our results suggest that if we use a strategy with a high rate of exchange between the trajectories, both the number of completed tours tends to decrease and the abandoned time tends to decrease. This property could be an important factor when deciding what strategy to use because in some scenarios it could be very important that the ARs make completed tours in the trajectories (for example in monitoring missions). Finally, it is worth noting that obtaining higher levels of communication in the team, that is, minimizing the broadcast time, depends on the topology of the graph as well as the selected strategy to use.

\section{Conclusions and Future Research}\label{sec:conclusion}


In this paper we introduce a strategy to synchronize a team of ARs and describe an approach to reestablish synchronization in the case when some ARs leave the system. We also introduce a new concept, that of starvation, in order to describe a phenomenon characterized by the permanent loss of synchronization for one or more active robots and propose various strategies to prevent it to ensure the fault-tolerance of the system.

 A solution in a simplified theoretical setting is presented first, and then adapted to a more realistic scenario.
 Our simulations suggest that this line of research is promising.  The work can be generalized to solve practical problems in several directions, including: (i) the trajectories can overlap and share multiple communication links that provide extra opportunities for information exchange; (ii) instead of failure, we can consider the inability of an agent to maintain its schedule along its trajectory; (iii) the implementation of real experiments in a multi-UAV testbed space.


\section*{Acknowledgments}
This work was initiated at the VI Spanish Workshop on Geometric Optimization, El Roc\'io, Huelva, Spain, held June 18-22, 2012. We thank the other participants of that workshop --C. Cort\'es, M. Fort, P. P\'erez-Lantero, J. Urrutia, I. Ventura and R. Zuazua-- for helpful discussions and contributing to a fun and creative atmosphere.

\section*{Funding}This work has been supported by the
ARCAS Project, funded by the European
Commission under the FP7 ICT Programme
(ICT-2011-287617) and the CLEAR Project
(DPI201 1-28937-C02-01), funded by the
Ministerio de Ciencia e Innovaci\'on of the
Spanish Government.

{\begin{minipage}[l]{0.2\textwidth}\includegraphics[scale=.15]{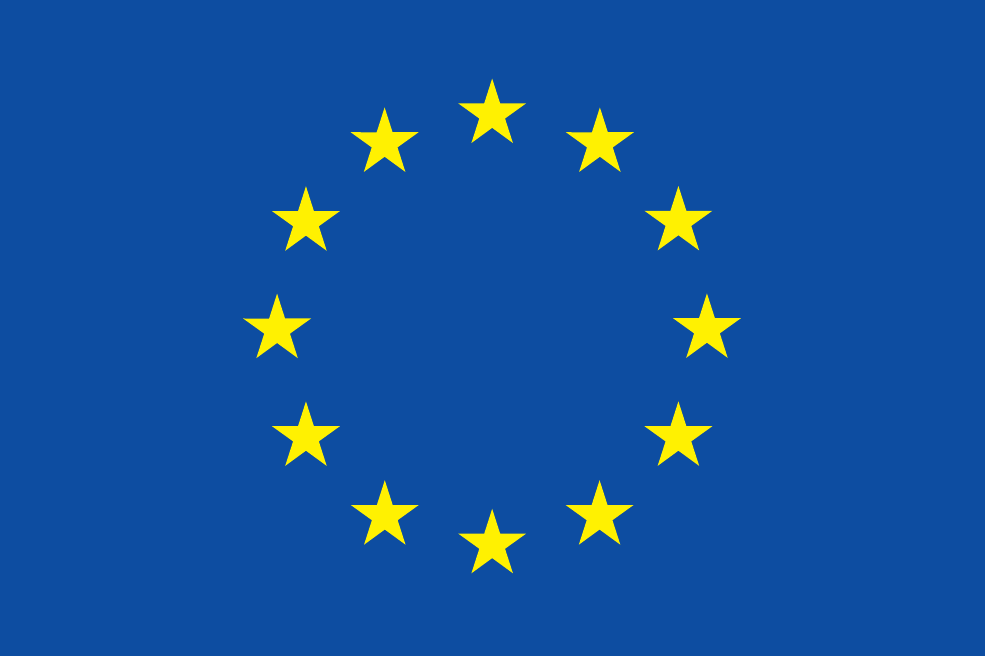} \end{minipage}%
	\begin{minipage}[l]{0.8\textwidth}
      This project has received funding from the European Union's Horizon 2020 research and innovation programme under the Marie Sk\l{}odowska-Curie grant agreement No 734922.
     \end{minipage}

\bibliography{bibliography}
\end{document}